%% file: main.tex
\documentclass{article}

\usepackage{collas2026_conference,times}

\usepackage{microtype}
\usepackage{graphicx}
\usepackage{booktabs} %

\usepackage{nicefrac}

\usepackage{algorithmic}
\usepackage{algorithm}

\input{math_commands.tex}

\usepackage{amssymb}
\usepackage{amsthm}
\usepackage{tikz-cd}
\usepackage{thmtools,thm-restate}
\usepackage{mathtools}

\usepackage{thm-restate}

\definecolor{ourMethod}{HTML}{0173b2}
\definecolor{LayerNorm}{HTML}{029e73}
\definecolor{Resetting}{HTML}{de8f05}
\definecolor{NaP}{HTML}{cc78bc}
\definecolor{spectral}{HTML}{ca9161}

\definecolor{snp}{HTML}{d55e00}
\definecolor{das}{HTML}{56b4e9}

\usepackage{xcolor}
\usepackage{subcaption}
\usepackage{caption}

\PassOptionsToPackage{hyphens}{url}
\usepackage{xurl}
\usepackage[
colorlinks,
citecolor=cyan,
linkcolor=blue,
urlcolor={blue!80!black}
]{hyperref}

\usepackage[capitalize,noabbrev]{cleveref}

\crefname{appendix}{appendix}{appendices}
\Crefname{appendix}{Appendix}{Appendices}
\crefname{assumption}{Assumption}{Assumptions}
\Crefname{assumption}{Assumption}{Assumptions}

\AddToHook{cmd/appendix/before}{\crefalias{section}{appendix}}

\Crefformat{section}{Section~#2#1#3}

\usepackage{xspace}

\newcommand{\tyler}[1]{\textcolor{blue}{\textbf{[Tyler: }#1\textbf{]}}}

\usepackage{amsthm}
\theoremstyle{plain}
\newtheorem{theorem}{Theorem}[section]
\newtheorem{proposition}[theorem]{Proposition}
\newtheorem{lemma}[theorem]{Lemma}
\newtheorem{corollary}[theorem]{Corollary}
\theoremstyle{definition}
\newtheorem{definition}[theorem]{Definition}
\newtheorem{assumption}{Assumption}[section]

\usepackage[framemethod=TikZ]{mdframed}
\usepackage{needspace}
\mdfdefinestyle{boxstyle}{%
 backgroundcolor=blue!6,
 roundcorner=8pt,
 innertopmargin=8pt,
 innerleftmargin=10pt,
 innerbottommargin=8pt,
 linewidth=0pt
}

\newcommand{\spectralclip}{SingularClip\xspace}

\newcommand*\samethanks[1][\value{footnote}]{\footnotemark[#1]}

\newif\ifincludeRL
\includeRLfalse

\title{SingularClip: Preventing Spectral Collapse to Maintain Plasticity in Continual and Reinforcement Learning}

\author{Tyler Kastner\thanks{Equal contribution. Correspondence to: \href{mailto:tkastner@cs.toronto.edu}{tkastner@cs.toronto.edu}.} \\
University of Toronto, Vector Institute, Mila \\
\And 
Nimrod De La Vega\samethanks \\
Polytechnique Montréal, Mila \\
\And 
Amir-massoud Farahmand \\
Polytechnique Montréal, Mila
}

\collasfinalcopy %

\begin{document}

\maketitle

\begin{abstract}
Neural networks trained on nonstationary tasks frequently lose the ability to fit new targets, a phenomenon referred to as loss of plasticity. We identify a novel source of plasticity loss due to the growing anisotropy of weight matrices' singular values during training, and analyze this phenomenon both empirically and theoretically. To mitigate this issue, we introduce \spectralclip, a procedure that periodically clips the singular values of all weight matrices. We show that \spectralclip performs strongly against baselines across a range of tasks in both continual supervised learning and deep reinforcement learning. 
\end{abstract}

\section{Introduction}
While deep learning excels at fitting static datasets, the ability to adapt rapidly in nonstationary environments remains a significant challenge. 
Networks trained on nonstationary data often become rigid, eventually losing the \emph{plasticity} required to fit new targets over time. This occurs in a range of settings, from learning new tasks in Continual Learning (CL) to tracking changing policies and value functions in Reinforcement Learning (RL). 

Previous work has identified various causes for the loss of plasticity, as well as techniques to prevent it.
Many of these can be interpreted as controlling the largest singular value of weight matrices, either implicitly by regularizing a choice of weight matrix norm, or directly regularizing the largest singular value \citep{yoshida2017spectral, bjorck2021towards, lewandowski2025learning}. At the same time, other works have attempted to preserve the plasticity properties present at network initialization, by regularizing towards the initial parameters \citep{kumar2023maintaining}, regularizing towards an orthogonal initialization \citep{chung2024parseval}, or periodically re-initializing the networks \citep{nikishin2022primacy}.

In this work, we identify a novel mechanism of plasticity loss, which emerges due to the increasing anisotropy of weight matrices' singular values. We term this \emph{anisotropy-induced plasticity loss} (AIPL). While previous work has identified that related quantities such as effective rank correlate with plasticity loss \citep{lyle2023understanding, kumar2023maintaining}, to our knowledge we are the first work to identify this anisotropy as a causal mechanism.
We show that this occurs empirically in both CL and RL settings, as well as in simplified theoretical settings. We then discuss how this leads to a loss of plasticity due to degraded gradient propagation, and provide theoretical lower bounds which quantify the slowdown of learning. Naturally, methods which only control the largest singular value do not prevent AIPL, as bounding the largest singular value allows for unbounded condition number growth, and we corroborate this finding empirically.

To address AIPL, we introduce \spectralclip, a method in which we periodically clip the singular values of all weight matrices. We demonstrate that this can be interpreted as a way to prevent AIPL while approximately maximizing the amount of information retained in the network.
When comparing to existing techniques for plasticity loss, we find that any method only regularizing the largest singular value does not prevent AIPL, as the anisotropy can still grow unbounded. Conversely, we show that methods such as resetting \citep{nikishin2022primacy}, which have become ubiquitous in reinforcement learning \citep{schwarzer2023bigger, kim2023sample, nauman2024bigger, lee2025simba}, are overly harsh and lose more information than needed, resulting in suboptimal learning and worsened sample complexity.

The organization of the paper is as follows. In \Cref{sec:bacgroundPlasticity}, we review previously introduced causes for loss of plasticity, and techniques introduced to mitigate them. In \Cref{sec:aipl}, we introduce AIPL as a cause of plasticity loss, and provide a description of what causes it to occur. In \Cref{sec:aipl-theory}, we provide a theoretical analysis of AIPL and obtain rates for the slowdown in plasticity incurred. In \Cref{sec:singularclip}, we introduce \spectralclip, and describe its implementation details, evaluating \spectralclip against baselines across continual supervised learning and deep reinforcement learning settings in \Cref{sec:empirical}, and lastly in \Cref{sec:related}, we discuss related work which may be relevant to the reader.

\section{Mechanisms of Plasticity Loss}\label{sec:bacgroundPlasticity}
Plasticity loss can arise from multiple seemingly unrelated factors \citep{lyle2024disentangling}. We highlight three mechanisms that have recently been characterized, along with mitigations introduced to address them.

\textbf{Dormant neurons.} Dormant neurons are activated for almost no inputs \citep{sokar2023dormant}. Because gradients vanish on these units, their weights stagnate, shrinking the effective capacity of the network. The proportion of dormant neurons has been observed to grow throughout training in nonstationary settings, prompting interventions which reset inactive units \citep{sokar2023dormant, liu2025measure}, or introduce normalization layers prior to nonlinearities to maintain activation \citep{lyle2024normalization}.

\textbf{Out-of-date optimizer statistics.} Adaptive optimizers such as Adam \citep{kingma2015adam} maintain exponential moving averages of gradient moments. In continual learning, gradients accumulated on earlier tasks may become misaligned, or even contradictory, to the current task, degrading adaptation. The proposed remedies include first-order momentum corrections \citep{bengio2021correcting}, carefully tuned hyperparameters that limit carry-over between tasks \citep{lyle2023understanding,dohare2023overcoming}, and periodic partial resets of the optimizer state \citep{asadi2023resetting,ellis2024adam}.

\textbf{Effective learning rate decay.} The presence of normalization layers causes parts of neural networks to become scale-invariant: $f(x; \theta)=f(x; c\theta)$ for any $c>0$. Consequently, gradients scale inversely with parameter norms as $\nabla f(x; c\theta)=\frac{1}{c}\,\nabla f(x; \theta)$, motivating the definition of an \emph{effective learning rate} $\alpha_{\mathrm{eff}}=\alpha/\|\theta\|_{F}$ \citep{van2017l2}. With non-infinitesimal steps, parameter norms grow during training \citep{van2017l2}, forcing $\alpha_{\mathrm{eff}}$ to decay and eventually stall learning. Mitigations include weight decay, $L_2$ regularization towards initialization \citep{kumar2023maintaining}, and periodic projection of the parameters onto a Frobenius-norm ball \citep{lyle2024normalization}.

\section{Anisotropy-Induced Plasticity Loss}\label{sec:aipl}

Empirically, it has been observed that gradients in deep learning often reside in low-dimensional subspaces of the parameter space \citep{gur2018gradient, larsen2022how, sonthalia2025low}. Under gradient descent-based training strategies, this causes weight matrices to concentrate their mass along these dimensions. From a spectral perspective, these matrices develop skewed singular value distributions, and high condition numbers. 
We now argue that this is typically not problematic in stationary settings if future gradient directions remain relatively stationary; however, it can cause plasticity loss in continual settings.

\begin{figure}[t]
 \centering
 \includegraphics[width=\linewidth]{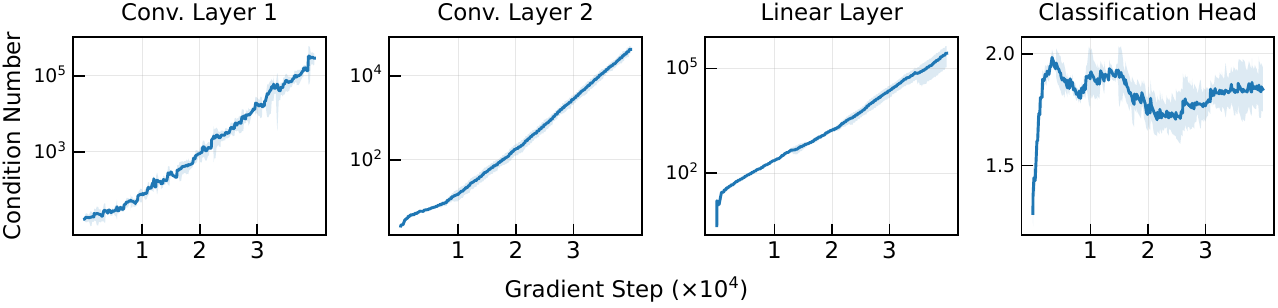}
 \caption{Evolution of the condition numbers for each layer in a network training on Random Label CIFAR task. Error bars show 95\% bootstrapped confidence intervals across 10 runs.}
 \label{fig:cifar-layer-conds}
\end{figure}

Before discussing how this leads to plasticity loss, we first empirically quantify the extent to which condition number growth occurs in standard tasks. In \Cref{fig:cifar-layer-conds}, we plot the condition numbers of each layer in a convolutional network trained on Random Label CIFAR (a description of this task can be found in \Cref{sec:empirical}). We find that the trend of all layers except the final classification head experiencing a roughly exponential growth of anisotropy occurs across architectures and tasks; we hypothesize that the lack of anisotropy in the classification head may be due to the softmax bottleneck phenomenon \citep{yang2018breaking, NEURIPS2018_9dcb88e0}. In \Cref{fig:esd-plot}, we plot the singular values of the first convolutional layer of a network trained on the Random Label CIFAR task over time; this allows one to visualize the evolution of all condition numbers instead of just the range shown by condition numbers.

\begin{figure}
 \centering
 \includegraphics[width=\linewidth]{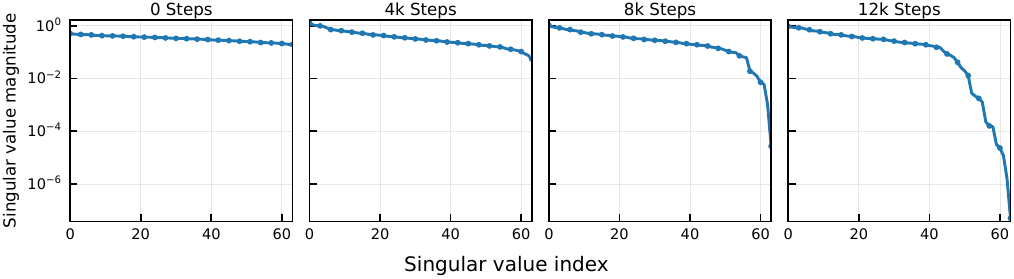}
 \caption{Evolution of the singular value distribution for the first convolutional layer in a run of Random Label CIFAR. Note that the largest singular value increases by a factor of roughly 3, as the parameter norm growth is controlled by weight decay.}
 \label{fig:esd-plot}
\end{figure}

With this empirical evidence that anisotropy grows throughout training, we now describe a mechanism for which this leads to a loss of plasticity. As training progresses, weight matrices will become anisotropic, and concentrate their mass in a certain \emph{dominant subspace}. In nonstationary settings, gradients for future tasks may have significant components outside this subspace (the intuition for this is as follows: the dominant subspace was formed by earlier gradient vectors, which may be unrelated to the current task's gradients). When this occurs, we argue that learning along these directions is slowed due to the imbalance in singular values. We name this phenomenon \emph{anisotropy-induced plasticity loss} (AIPL).

To illustrate why this slowdown occurs, consider a single linear layer in an otherwise arbitrary neural network with weight matrix $W \in \mathbb{R}^{m\times n}$. We write its singular value decomposition as $W =U \Sigma V^\top$, where $\Sigma = \mathrm{diag}(\sigma_1,\dots,\sigma_d)$. From our earlier observations, we will assume that $(\sigma_1,\dots,\sigma_d)$ exhibits anisotropy (so that $\sigma_1\gg \sigma_d$). 
Suppose the activation gradient of this layer's output is given by $g \in \mathbb{R}^m$. Let us write $ (\alpha_1,\dots, \alpha_m)$ for the coordinates of $g$ in the basis of the columns of $U$. The gradient signal passed downstream to all preceding layers in the network is given by $W^\top g \in \mathbb{R}^n$, which can be written as $\sum_{i}\alpha_i \sigma_iv_i$. If the current gradients are not aligned with the dominant subspace, meaning $\alpha_i$ is not concentrated in the dimensions of the dominating singular values $\sigma_i$, the downstream gradient passed to earlier layers is attenuated due to the skewness. This then leads to a loss of plasticity, as the attenuation diminishes the learning signal received by the network to adapt to the current task. This computation also illustrates why the phenomenon is benign in stationary settings: in cases where the current gradient is aligned with those previously encountered (and hence with the current weight matrix), the gradients are propagated without reduction.

Having established a process through which anisotropy leads to plasticity loss, we now empirically analyze this phenomenon. Following the argument in the previous paragraph, we note that anisotropy has two components: the amount of magnitude difference between singular values, which affects the magnitude of gradient attenuation, and the number of small singular values, which dictates the number of update directions which are attenuated. We begin with a simplified stationary MNIST setting, where we modify the anisotropy of network initialization to isolate its effect on learning time. Specifically, for a parameter $d$, we initialize all weight matrices as $U\Sigma_d V^\top$, where $U$ and $V$ are randomly sampled orthogonal matrices, the top $d$ singular values of $\Sigma_d$ are set to 1, and the remaining are sampled at scale $1\times 10^{-5}$. We rescale each input matrix to constant Frobenius norm to eliminate initialization norm scale as a confounder. For each $d$, we measure the number of steps needed to achieve perfect train accuracy on a subset of MNIST, and plot the results in \Cref{fig:cond-causal-plots}. We can see that the number of steps required experiences an inverse monotonic relationship with $d$, demonstrating a clear effect of anisotropy on plasticity (note that a smaller $d$ corresponds to a larger degree of anisotropy). We next consider the performance of a network on the Random Label MNIST setting (further details can be found in \Cref{sec:empirical}). This setting is separated into a sequence of discrete tasks, where each task is an independent learning problem. We track the average train accuracy across each task as a measure of plasticity, as this directly captures the network's ability to accurately fit the data it is observing, as well as how quickly it can fit the data. We plot the initial condition number against the average accuracy obtained in \Cref{fig:cond-causal-plots}, which demonstrates a clear negative relationship between the condition number and the plasticity of the network.

\begin{figure}[h]
 \centering
 \begin{subfigure}[t]{0.45\linewidth}
 \centering
 \includegraphics[width=\linewidth]{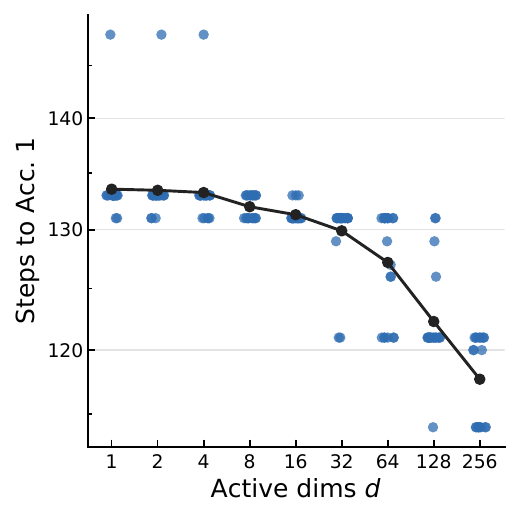}
 \end{subfigure}
 \begin{subfigure}[t]{0.45\linewidth}
 \centering
 \includegraphics[width=\linewidth]{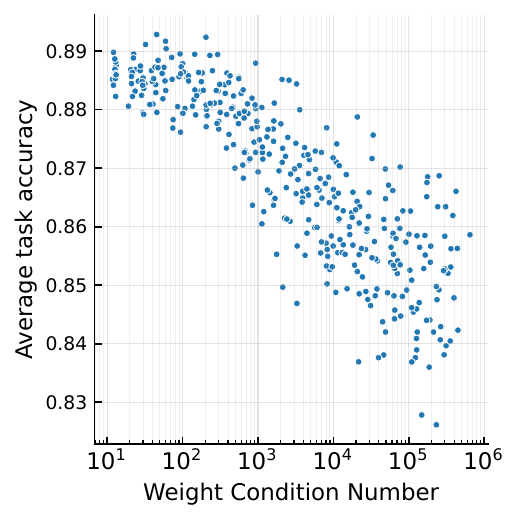}
 \end{subfigure}
 \caption{Measuring the number of steps to reach perfect train accuracy as a function of initialization anisotropy, across 20 independent runs (left). Measuring initial weight condition number against average task accuracy in the Random Label MNIST task; datapoints are collected in a sequence of 35 tasks, across 10 independent runs (right).}
 \label{fig:cond-causal-plots}
\end{figure}

\section{Theoretical Analysis of AIPL}\label{sec:aipl-theory}

We now study AIPL in a simplified theoretical setting, where its effect can be isolated from the many confounding factors found in deep learning in practice.
 
We consider a two-layer linear network $f_W(x)=W_2W_1x$, whose goal is to minimize the cross-entropy loss to a teacher model, where $W_1 \in \mathbb{R}^{r \times d}$, $W_2 \in \mathbb{R}^{K \times r}$. This is a common setup for the theoretical analysis of neural network phenomena \citep{saxe2014exact, kawaguchi2016, soudry2018, arora2018optimization, lyu2024dichotomy}. We consider the evolution of the weights under the gradient flow of the cross-entropy loss, projected to remain on the hypersphere: $\|W_1(t)\|_F=\|W_2(t)\|_F=1$. 

There are multiple benefits of considering these projected dynamics as opposed to the standard gradient flow. Firstly, it corresponds to the gradient flow obtained when activation normalization is present after each layer (the dynamics are unchanged for any normalization which induces scale-invariance, such as batch normalization \citep{ioffe2015batch}, weight normalization \citep{salimans2016weight}, layer normalization \citep{ba2016layer}, RMS normalization \citep{zhang2019root}), and as discussed in \Cref{sec:bacgroundPlasticity}, normalization layers are an important plasticity loss mitigation. Secondly, by maintaining constant Frobenius norm dynamics, we remove any plasticity loss arising from parameter norm growth, as also discussed in \Cref{sec:bacgroundPlasticity}. Thirdly, any parameter change in the radial direction does not affect the anisotropy of the parameters, by ignoring such changes, we isolate the effect of anisotropy on plasticity dynamics.

Our goal is to quantify the time $T$ to solve the network's current task. We do this assuming a fixed, possibly anisotropic initialization, with the aim that the slowdown in task completion time is proportional to the amount of anisotropy. Hence, this is the theoretical setting corresponding to the Random Label MNIST empirical analysis done in \Cref{fig:cond-causal-plots}. We begin with introducing a measure of mass assigned by the network $f_W$ to hidden subspaces.

\begin{definition}
 Let $\mathcal U\subseteq\mathbb R^r$ be a subspace, and let $P_{\mathcal U} \in \mathbb{R}^{r \times r}$ be the orthogonal projector onto $\mathcal U$. For parameters $W=(W_1, W_2)$, we define
\[
M_{\mathcal U}(W) := \|P_{\mathcal U}W_1\|_F^2 + \|W_2P_{\mathcal U}\|_F^2.
\]
\end{definition}

 Intuitively, the first term measures the amount of mass of $W_1$ whose output lies in $\mathcal{U}$, and the second term measures the amount of mass of $W_2$ which reads input from $\mathcal U$. 
 We next introduce the following quantities, which allow us to upper-bound $\mathcal{M}_\mathcal{U}(W)$ at initialization:
\[\sigma_{1, \mathcal{U}} = \|P_{\mathcal U}W_1(0)\|_*,\quad n_1=\mathrm{rank}(P_{\mathcal U}W_1(0)), \quad \sigma_{2, \mathcal{U}} = \|W_2(0) P_{\mathcal U}\|_*,\quad n_2=\mathrm{rank}(W_2(0)P_{\mathcal U}). \]
Intuitively, note that if $\mathcal{U}$ was exactly spanned by left singular vectors of $W_1$, then $n_1$ would be the number of singular vectors, and $\sigma_{1,\mathcal U}$ would be the largest corresponding singular value of this set (likewise for right singular vectors of $W_2$). With these definitions in place, we present our main theoretical result.

\begin{restatable}{theorem}{main}
\label{thm:time-lb-informal}
Suppose there exists $m>0$ such that $f_W$ solving the current task requires that $M_\mathcal{U}(W) \ge m$, and at initialization $\mathcal{M}_\mathcal{U}(W(0)) < m$. If the weights $W(t)$ are updated using projected gradient flow, let $T$ be the first time which $W(T)$ solves the current task. Then we have
\[
T \gtrsim \log\!\left(\frac{m}{n_1\sigma_{1,\mathcal{U}} + n_2\sigma_{2,\mathcal{U}}}\right).
\]
In particular, if $\mathcal{U}$ perfectly aligns with the smallest singular subspace of each weight matrix, then
\[
T \gtrsim \log(m \cdot\min(\kappa_1, \kappa_2)),
\]
where $\kappa_1$ and $\kappa_2$ are the condition numbers of $W_1$ and $W_2$ respectively.
\end{restatable}

\Cref{thm:time-lb-informal} quantitatively characterizes the impact of weight anisotropy on network plasticity. In particular, it states that for a new task (where we kept the task itself and the condition for solving it arbitrary to preserve generality), the adaptation time grows logarithmically with the inverse of the largest singular values of the subspaces which must be populated, hence populating a subspace with low singular values will result in a logarithmic proportional to the degree of anisotropy. This can be interpreted in the empirical results of \Cref{sec:aipl}: \Cref{fig:esd-plot} demonstrates the emergence of directions with small singular values, as well as the increase in the number of such directions. The condition number plots in \Cref{fig:cifar-layer-conds} further demonstrate the worst-case amount of slowdown, when the adaptation requires populating the smallest singular subspace. This result allows us to establish a causal relationship, albeit in a simplified theoretical setting.

\section{SingularClip: Maintaining Plasticity via Spectral Control}\label{sec:singularclip}

To combat the negative effects of AIPL, we propose a simple solution: periodically clipping the singular values of all weight matrices in the network to a predetermined range. By controlling the entire range of singular values, we prevent the anisotropy from growing without bound. Importantly, this is not the case for methods which only constrain the largest singular value: the anisotropy and condition number can still grow unbounded in this case. Resetting resolves this issue, but loses \emph{all} information learnt in the weight matrix. Alternatively, we will show that our form of spectral clipping is in some sense a \emph{minimal} intervention to achieve this objective.

Formally, the operation we will be applying to the weight matrices is
\[
W\mapsto \mathrm{sc}_a^b(W) \coloneq U\, \mathrm{clip}(\Sigma,a,b)\, V^{\top},
\]
where $W = U\Sigma V^\top$ is its singular value decomposition and $\mathrm{clip}(X,a,b)$ clips the entries of a matrix $X$ to be within $[a,b]$. We will further consider a single hyperparameter $c$, and perform singular clipping to the range $[\nicefrac{1}{c}, c]$. This in turn presents a simple intuition: after performing \spectralclip, the linear transformation performed by each weight matrix preserves all rotations, but can only expand or contract inputs by a factor of at most $c$.

\begin{algorithm}[tb]
 \caption{\spectralclip}
 \label{alg:themethod}
\begin{algorithmic}
 \STATE {\bfseries Input:} parameters $\theta$ of network; clip ratio $c \geq 1$; clipping period $K$ gradient steps.
 \FOR{training step $t = 1,2,\dots$}
 \STATE Perform optimization step on $\theta$
 \IF{$t \bmod K = 0$}
 \FOR{each linear and convolutional layer $W_\ell$ in $\theta$}
 \IF{$W_\ell$ is a convolutional layer}
 \STATE $W_\ell \gets \mathrm{reshape2D}(W_\ell)$
 \ENDIF
 \STATE Compute SVD: $W_\ell = U \Sigma V^\top$
 \STATE Clip singular values: $\Sigma' \gets \mathrm{clip}(\Sigma, \nicefrac1c, c)$
 \STATE Apply clipping: $W_\ell \gets U \Sigma' V^\top$
 \IF{$W_\ell$ is a convolutional layer}
 \STATE Restore shape: $W_\ell \gets \mathrm{unreshape}(W_\ell)$
 \ENDIF
 \ENDFOR
 \ENDIF
 \ENDFOR
\end{algorithmic}
\end{algorithm}

We now formalize a certain sense in which this intervention is minimal.
\begin{restatable}{proposition}{FrobeniusProjection}\label{prop:proj}
The matrix $\mathrm{sc}_a^b(W)$ is the Frobenius norm projection of $W$ onto the set of matrices with singular values in $[a,b]$. Equivalently, 
\[ \left\| \mathrm{sc}_a^b(W) - W \right\|_F \leq \left\| B - W \right\|_F \]
for any matrix $B$ with all singular values in $[a,b]$. (See \Cref{sec:singularclip-proofs} for the formal proof).
\end{restatable}
As a Frobenius projection, SingularClip minimizes an upper bound on feature distortion: for any vector $x$ and matrices $W, W'$, we have $\|Wx - W'x\|_2 \le \|W - W'\|_F \|x\|_2$. This provides a guarantee on the extent that \spectralclip perturbs the outputs of the layer.

\subsection{Algorithmic Details}

We now discuss some details regarding the implementation of \spectralclip (\Cref{alg:themethod}). We perform the clipping by computing the SVD numerically and clipping it directly; we discuss alternatives in \Cref{sec:algo-details-app}.

\subsubsection{Hyperparameters}
\label{sec:hyperparams}

SingularClip introduces two hyperparameters: the \emph{clip ratio} $c \geq 1$, and the \emph{clipping period} $K$ (the number of gradient steps between successive clips). We briefly discuss setting both of these. 

\paragraph{Setting $c$.}
The parameter $c$ determines how much anisotropy a layer is allowed to represent. When $c=1$, SingularClip reduces to exact orthogonalization; as $c \to \infty$, it becomes the identity. Smaller values of $c$ more aggressively prevent anisotropy growth, but also perturb the learned transformations more and can cause a larger performance drop after clipping. Larger values preserve more of the current network, but may not sufficiently mitigate plasticity loss. We provide a sweep over the choice of $c$ in various settings in \Cref{sec:algo-details-app}.

\paragraph{Setting $K$.}
The parameter $K$ affects two aspects of the algorithm: how much anisotropy can accumulate in the network before clipping is applied, and the computational overhead of performing \spectralclip. We do not tune $K$, and instead set it based on the problem at hand: for continual learning experiments, we set it to the length of each task, and for reinforcement learning experiments, we match the resetting frequency of the baselines we compare against.

\subsubsection{On Two-Sided Clipping}

We design the algorithm to perform two-sided clipping; however, one may wonder whether one-sided clipping is sufficient. We first note that performing only lower-clipping will not prevent parameter norm growth, which will lead to plasticity loss due to effective learning rate decay (\Cref{sec:bacgroundPlasticity}). Conversely, performing only upper-clipping will prevent effective learning rate decay, but will not prevent AIPL, since the lower singular values can decrease without control. We provide ablations of one-sided clipping in \Cref{sec:algo-details-app} which support our arguments here.

\section{Empirical Results}\label{sec:empirical}
We now evaluate \spectralclip across a collection of continual supervised learning and reinforcement learning settings. Our main goal of this evaluation is to empirically study two hypotheses: (i) \spectralclip maintains plasticity better than methods which regularize the largest singular value only, and (ii) \spectralclip is able to better maintain task-relevant information than resetting. 
We separate our empirical analysis into two parts, experiments which focus on continual supervised learning settings, and deep reinforcement learning settings, as these are where plasticity loss has been found to occur most prevalently.
For continual learning experiments, we perform \spectralclip with $c=2$, and for reinforcement learning experiments, we use $c=4$, based on our sweeps in \Cref{sec:algo-details-app}. We intentionally did not tune $c$ at a per-task level in order to provide a fair comparison against baselines.
Further details can be found in \Cref{sec:exp-details}.

\subsection{Continual Learning Benchmarks}
We consider a collection of five continual learning tasks which are commonly used in the literature. We provide a brief description of each task below.

\input{widefig}

\begin{itemize}
 \item \textbf{Permuted MNIST} \citep{goodfellow2013empirical}: We sample 10,000 images from the MNIST training dataset. For each task, we fix a uniformly sampled permutation matrix, and apply this to the flattened images. The model's goal is to minimize the cross-entropy loss between the permuted images and the true labels.

 \item \textbf{Random Label MNIST} \citep{lyle2024normalization}: We fix a subset of 2,048 MNIST images and assign each one a uniformly sampled label from $\{1,\dots,10\}$. The network's goal is to memorize these random labels, which are randomly reassigned in each subsequent task.
 
 \item \textbf{Random Label CIFAR} \citep{zhang2017understanding}: We repeat the task setting of Random Label MNIST with CIFAR-10 images.
 
 \item \textbf{Random Label ImageNet} \citep{lewandowski2025learning}: We repeat the task setting of Random Label MNIST but with 16,384 ImageNet images whose class labels are randomly sampled from $\{1,\dots,1000\}$.
 
 \item \textbf{Continual ImageNet} \citep{dohare2023overcoming}: Each task samples two ImageNet classes without replacement. The learner performs binary classification for 500 successive tasks.
\end{itemize}

We chose these problems as they allow us to study the effect of different types of nonstationarity, as well as various architectures. 
Depending on the type of nonstationarity, each task requires different parts of the network to be re-learnt. For example, Permuted MNIST requires relearning the entire representation used for each task, while random label tasks can be solved exactly by maintaining the representation and learning a permutation of the final classification layer. For the architectures, MNIST-based experiments used a fully-connected network which operates on flattened image inputs, while CIFAR-10 and ImageNet-based experiments used a CNN architecture (see \Cref{sec:exp-details} for further details). 

For each setting, we repeat the experiment for 100 tasks, except for Continual ImageNet, as it runs for 500 tasks by definition. We report the online average train accuracy achieved by the network for each task, that is
\begin{equation}\label{eq:online-acc}
 A(T_i)=\frac1N\sum_{t=1}^{N} a_t,
\end{equation}
where $a_t$ is the accuracy on batch $t$ of the task $T_i$, and task $T_i$ lasts for $N$ steps. We choose this measure as we find it a clear measure of a network's plasticity: note that two baselines can reach the same final train accuracy, however the more plastic one will do so in less steps: this is reflected in online train accuracy, but not the final accuracy achieved. We further only consider train accuracy in this work, as this is the most reflective of a network's plasticity, that is, its ability to fit the training data it is receiving. Furthermore, generalization performance is not well-defined for many of the tasks considered, such as memorizing random labels. 

For each benchmark, we compared \spectralclip to six baselines: normalization layers, periodic resetting \citep{nikishin2022primacy}, normalize and project (NaP) \citep{lyle2024normalization}, spectral regularization \citep{lewandowski2025learning}, shrink \& perturb \citep{ash2020warm}, and DASH \citep{shin2024dash}. In the context of our hypotheses, these baselines were chosen as follows: normalization layers provide a base level of plasticity intervention through the use of a now-standard architecture. NaP and spectral regularization allow us to test whether \spectralclip is superior to methods which only regularize the largest singular value; these methods control the largest singular value by explicitly regularizing it (spectral regularization) or implicitly regularize it by projecting the Frobenius norm of weight matrices. Comparing against resetting allows us to capture how much performance is gained by preserving properties of the weight matrices as described in \Cref{sec:singularclip}. Comparing against shrink \& perturb and DASH allow us to measure our performance against empirically well-performing baselines.

We present the results of these experiments in \Cref{fig:cl}. Across the board, we find that \spectralclip outperforms all baselines considered. Compared to the baselines, \spectralclip experienced minimal loss of plasticity, observable in the other algorithms through either an onset of decreasing performance, or an early plateauing of performance.
Resetting experienced very little loss of plasticity -- indeed any which does occur would be due to non-weight-related effects, such as out-of-date optimizer statistics as discussed in \Cref{sec:bacgroundPlasticity}. Resetting performed worst in benchmarks which were too complicated to solve within a single task, such as Random Label CIFAR and ImageNet. This supports our second hypothesis: \spectralclip was able to maintain task-relevant information and greatly outperformed resetting here.

\subsection{Reinforcement Learning}

We next consider reinforcement learning across two settings: the experimental setup of \citet{nikishin2022primacy}, where we consider a baseline soft actor-critic (SAC) agent and compare its performance with periodic resetting and performing \spectralclip; and the BRO agent \citep{nauman2024bigger}, an RL agent which combines many recent improvements in the field, such as a modern architecture, distributional reinforcement learning, optimistic exploration, and notably periodic resetting. We modify the BRO agent to apply \spectralclip in place of resetting in order to test our second hypothesis directly, that is, whether \spectralclip can maintain more task-relevant information than resetting.

\subsubsection{SAC Experiments}

We evaluate a SAC agent on the DeepMind Control Suite (DMC) \citep{tassa2018deepmind} with periodic interventions every $2\times10^{5}$ environment steps, following \citet{nikishin2022primacy}. We sweep the update-to-data (UTD) ratio (a hyperparameter dictating how many gradient steps are taken for every environment step) across $\{1,4,8\}$. We include layer normalization in our network as recommended by \citet{hussing2024dissecting}, resulting in our baseline SAC agent appearing stronger than the baseline reported in \citet{nikishin2022primacy}.

\begin{figure}[t]
 \centering
 \begin{subfigure}[t]{0.45\linewidth}
 \centering
 \includegraphics[width=\linewidth]{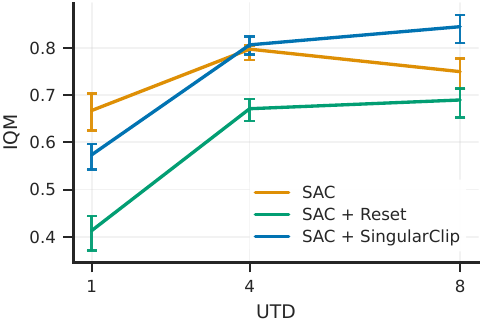}
 \end{subfigure}
 \begin{subfigure}[t]{0.5\linewidth}
 \centering
 \includegraphics[width=\linewidth]{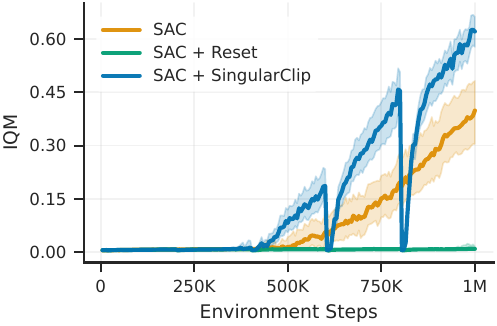}
 \end{subfigure}
 \caption{Performance of SAC-based agents aggregated across DMC suite. Performance is aggregated across 10 runs per task. Error bars indicate 95\% bootstrapped confidence intervals (left). Performance curves of SAC-based agents in the \texttt{humanoid-stand} environment averaged across 10 independent runs. Shaded regions correspond to 95\% bootstrapped confidence intervals (right).}
 \label{fig:sac-rl-plots}
\end{figure}

We present the final aggregated interquartile mean (IQM, \citep{Precipice2021}) scores of the agents across the suite in \cref{fig:sac-rl-plots}. 
Notably, we find that at UTD 1 the baseline SAC agent outperforms both periodic resetting and periodic \spectralclip, indicating that the loss of plasticity is not sufficient to require that amount of intervention in this regime. However, the performance of both resetting and \spectralclip scales much more favourably than that of the baseline SAC agent with respect to UTD: this is due to the increased prevalence of plasticity loss in higher UTD regimes. At all UTD levels, \spectralclip reaches an IQM level significantly higher than resetting, indicating that \spectralclip-based agents were able to retain more task-relevant information, and eventually attain a higher score.

In \Cref{fig:sac-rl-plots} we present the learning curves for the \texttt{humanoid-stand} environment, at UTD ratio 1. We chose this setting in particular, as it presents a striking example of our previous argument. The environment itself is relatively difficult, and for any agent it takes roughly 400,000 environment steps to receive any reward. Due to the harshness of the resets, the SAC + Reset agent is unable to receive any positive reward across the entirety of training. On the other hand, the agent with periodic \spectralclip is able to retain enough information in the weights to achieve positive reward, and even outperform the baseline SAC agent.

\subsubsection{BRO Experiments}

We evaluate the BRO agent \citep{nauman2024bigger} against a modified version of the agent which replaces all resets with \spectralclip applied at the same timesteps, and we keep all other hyperparameters the same. We evaluate these agents across the entirety of the DMC suite, and present the results in \Cref{fig:BRO-DMC-1}.

\begin{figure}[!ht]
 \centering
 \includegraphics{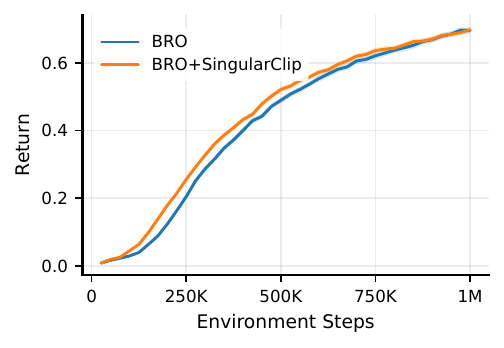}
 \caption{Performance of BRO-based agents across the DMC suite, across 20 independent runs. Shaded areas indicate bootstrapped 95\% confidence intervals.}
 \label{fig:BRO-DMC-1}
\end{figure}

This modification results in statistically significant (non-overlapping confidence intervals) performance improvements up to approximately 800,000 environment steps. 
The improvement ending at this point can be explained by the fact that BRO performs its final reset at 750,000 environment steps, so at this point it is possible the resetting-based models have performed enough gradient steps to fit the current replay buffer, and both networks achieve similar performance.
Overall, we emphasize the strong and significant improvement achieved by simply replacing the resetting with \spectralclip, achieved with a minimal change to the algorithm itself. This provides further evidence for our second hypothesis: by retaining task-specific knowledge over resetting, we are able to achieve sample complexity gains, despite all other modern improvements present in the BRO agent.

\section{Related Work}\label{sec:related}

A growing line of work aims to mechanistically explain loss of plasticity.
In reinforcement learning, capacity loss has been studied as a failure mode induced by nonstationary targets, with mitigation via regularization that anchors network outputs to their initialization \citep{lyle2022capacityloss}. Empirically, \citet{lyle2023understanding} provide a broad characterization of plasticity loss and connect it to changes in curvature over training, motivating interventions that preserve trainability. Related curvature-centric accounts argue that loss of plasticity can be viewed as a reduction in usable curvature directions during training \citep{lewandowski2023directions}. \citet{he2025spectralcollapse} identify Hessian spectral collapse at task boundaries and introduce a rank-based notion of trainability that unifies several plasticity-preserving strategies through the lens of curvature preservation. \citet{tang2025churn} analyze loss of plasticity in continual RL, proposing mechanisms to reduce churn and stabilize training dynamics. \citet{joudaki2025barriers} develop a dynamical-systems perspective in which loss of plasticity corresponds to gradient trajectories becoming trapped near stable manifolds yielding a first-principles account of why certain non-plastic states are difficult to escape. Related theoretical treatments of continual learning under gradient descent include \citet{jung2025convergence} and \citet{taheri2025theory}.

\citet{balestriero2025lejepa} highlight issues caused by anisotropic representations in the context of world modelling, and introduced a SigReg loss to regularize the representations towards an isotropic Gaussian. \citet{pasand2026stable} apply the SigReg loss to deep reinforcement learning. \citet{cisse2017parseval} introduce Parseval regularization, a method to regularize weight matrices towards orthogonality, to improve robustness towards adversarial examples. \citet{chung2024parseval} applied Parseval regularization to continual learning. \citet{han2026fire} replaced the Parseval regularization loss with periodic orthogonalization. \citet{palenicek2025xqc} demonstrated that Hessian conditioning affects performance in reinforcement learning, and that explains the performance of certain architectures such as normalization and categorical losses. \citet{lee2025hyperspherical} introduced a deep RL agent which improved performance by projecting weights to the hypersphere at each timestep, similar to the procedure of NaP.

\section{Conclusion}

In this paper, we identified a new mechanism of plasticity loss in deep neural networks, caused by the growth of weight matrix anisotropy throughout learning. We demonstrated this phenomenon empirically, and theoretically analyzed it in a simplified setting. We introduced a method, \spectralclip, to mitigate its effects and obtained strong empirical performance across a range of supervised learning and reinforcement learning settings.

\section*{Acknowledgments}

We acknowledge the support of NSERC through the Discovery Grant program [2021-03701], Polytechnique Montréal’s PIED program, and IVADO R$^3$AI grant. Resources used in preparing this research were provided, in part, by the Province of Ontario, the Government of Canada through CIFAR, and companies sponsoring the Vector Institute, as well as the compute resources provided by Mila.

\newpage

\bibliography{main}
\bibliographystyle{icml2025}

\input{appendix-new}

\end{document}

%% file: math_commands.tex
\usepackage{amsmath,amsfonts,bm}

\def\eqref#1{equation~\ref{#1}}

\def\1{\bm{1}}

\DeclareMathAlphabet{\mathsfit}{\encodingdefault}{\sfdefault}{m}{sl}
\SetMathAlphabet{\mathsfit}{bold}{\encodingdefault}{\sfdefault}{bx}{n}

%% file: widefig.tex
\begin{figure}[!t]
  \centering
    \begin{subfigure}[t]{0.32\textwidth}
    \centering
    \includegraphics[width=\linewidth]{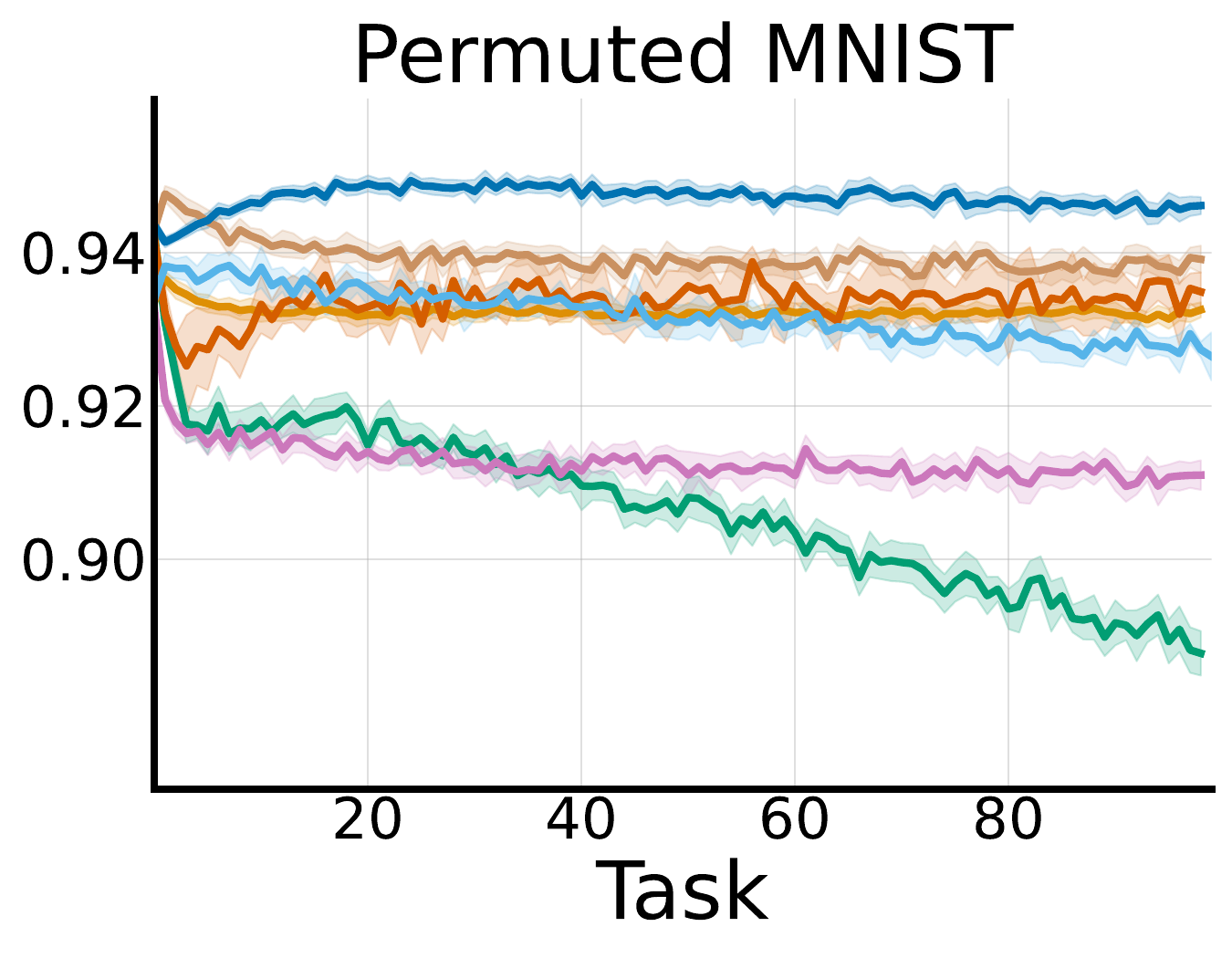}
  \end{subfigure}
  \begin{subfigure}[t]{0.32\textwidth}
    \includegraphics[width=\linewidth]{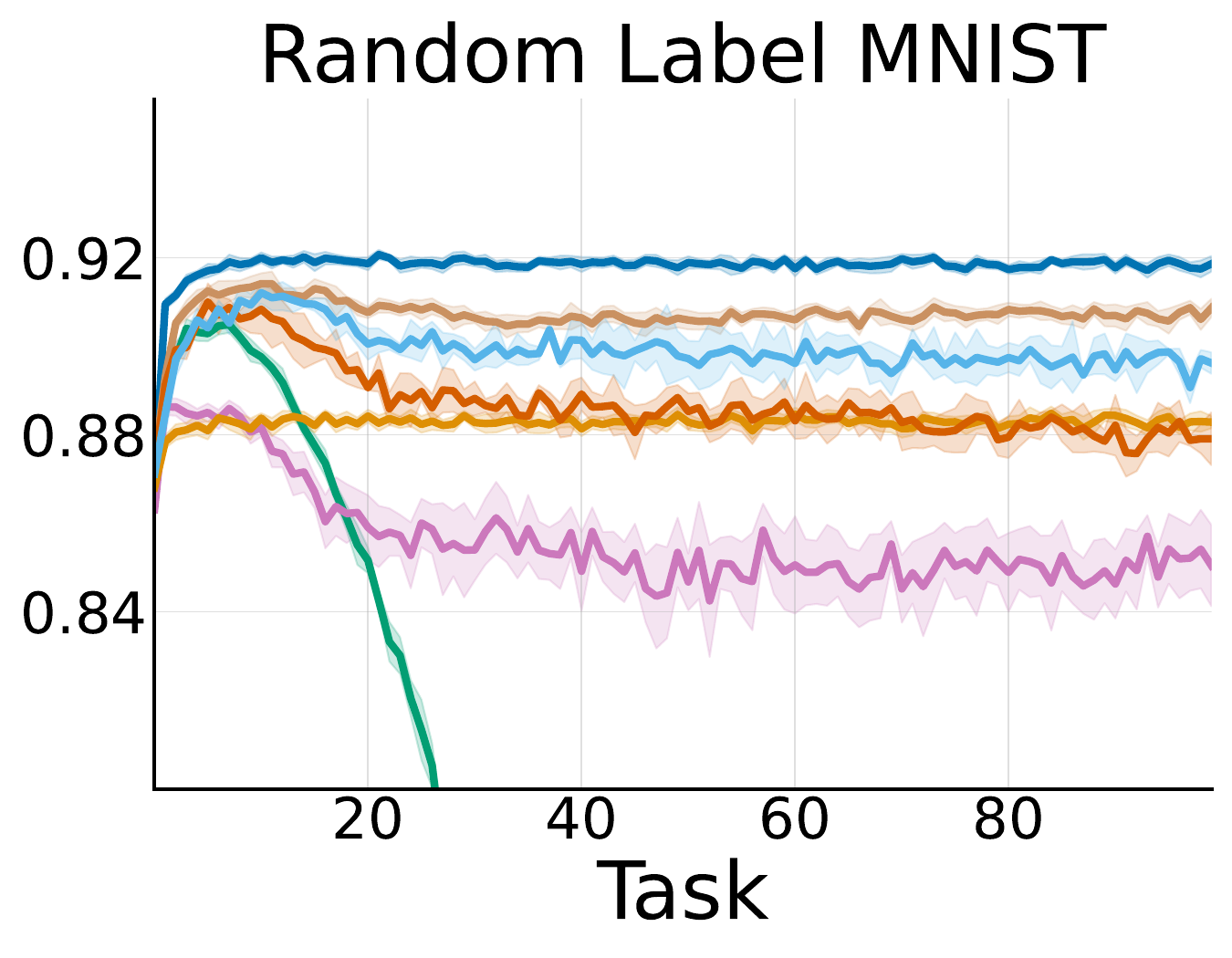}
  \end{subfigure}
  \begin{subfigure}[t]{0.32\textwidth}
    \includegraphics[width=\linewidth]{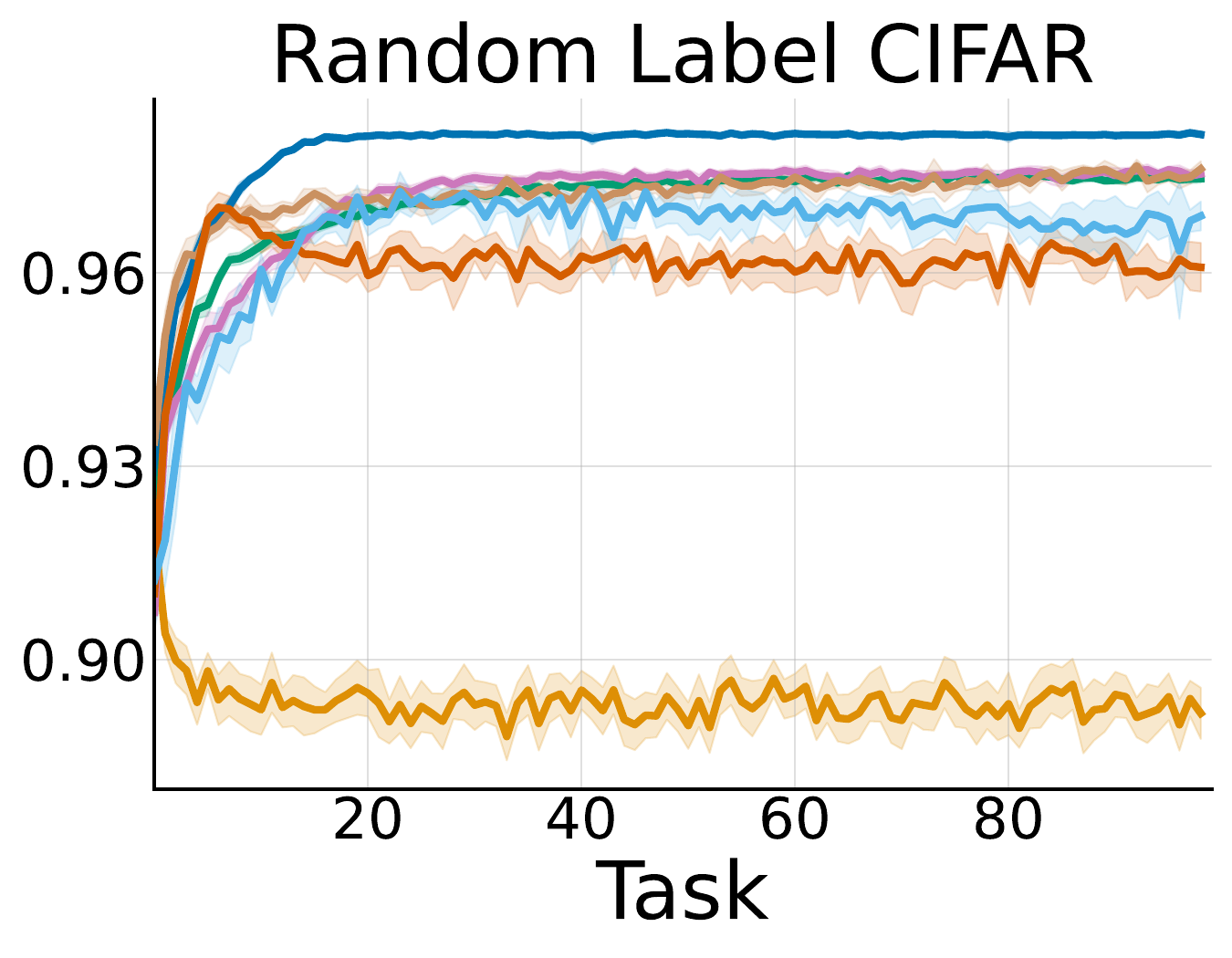}
  \end{subfigure}
  
  \begin{subfigure}[t]{0.32\textwidth}
    \centering
    \includegraphics[width=\linewidth]{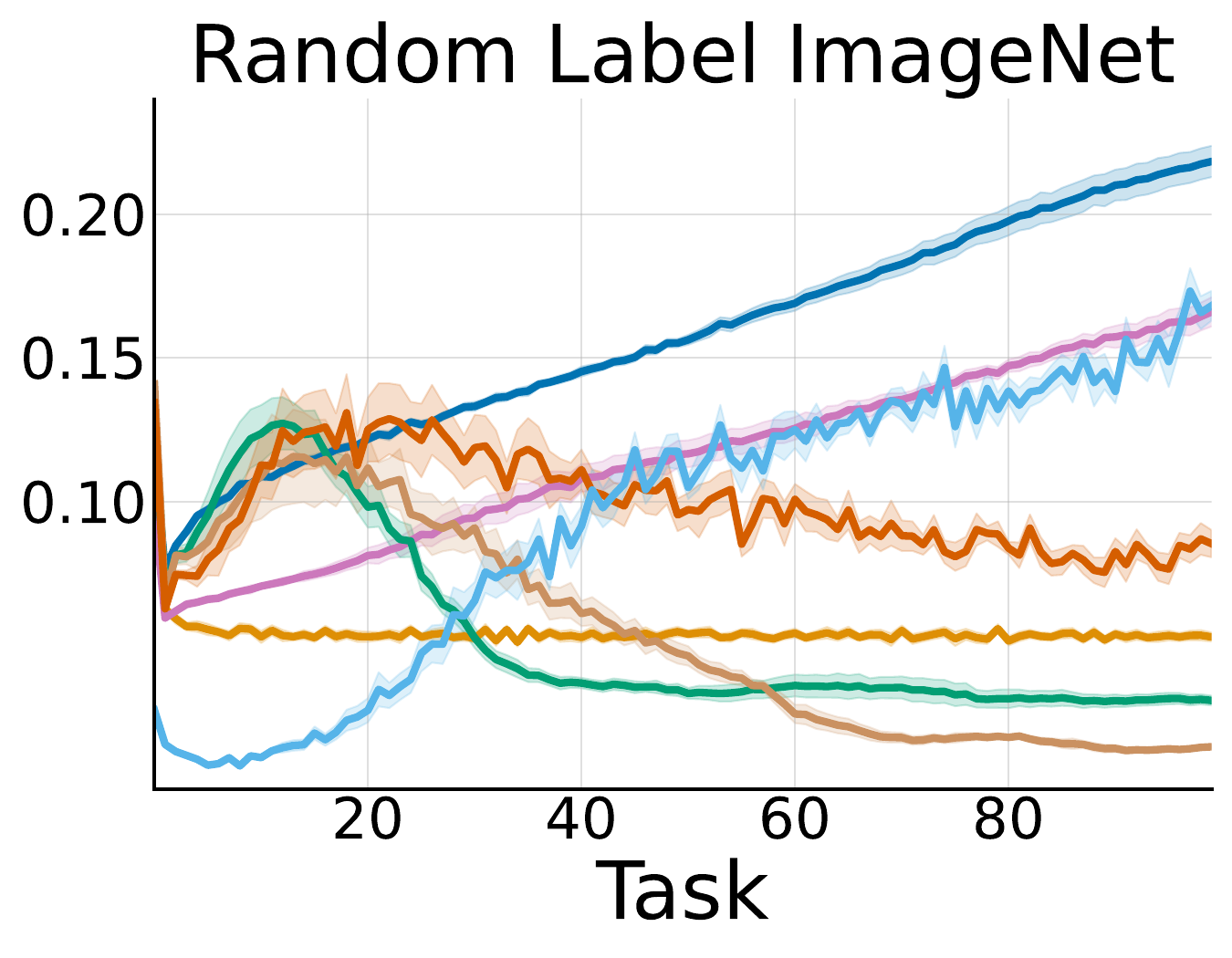}
  \end{subfigure}
  \begin{subfigure}[t]{0.32\textwidth}
    \centering
    \includegraphics[width=\linewidth]{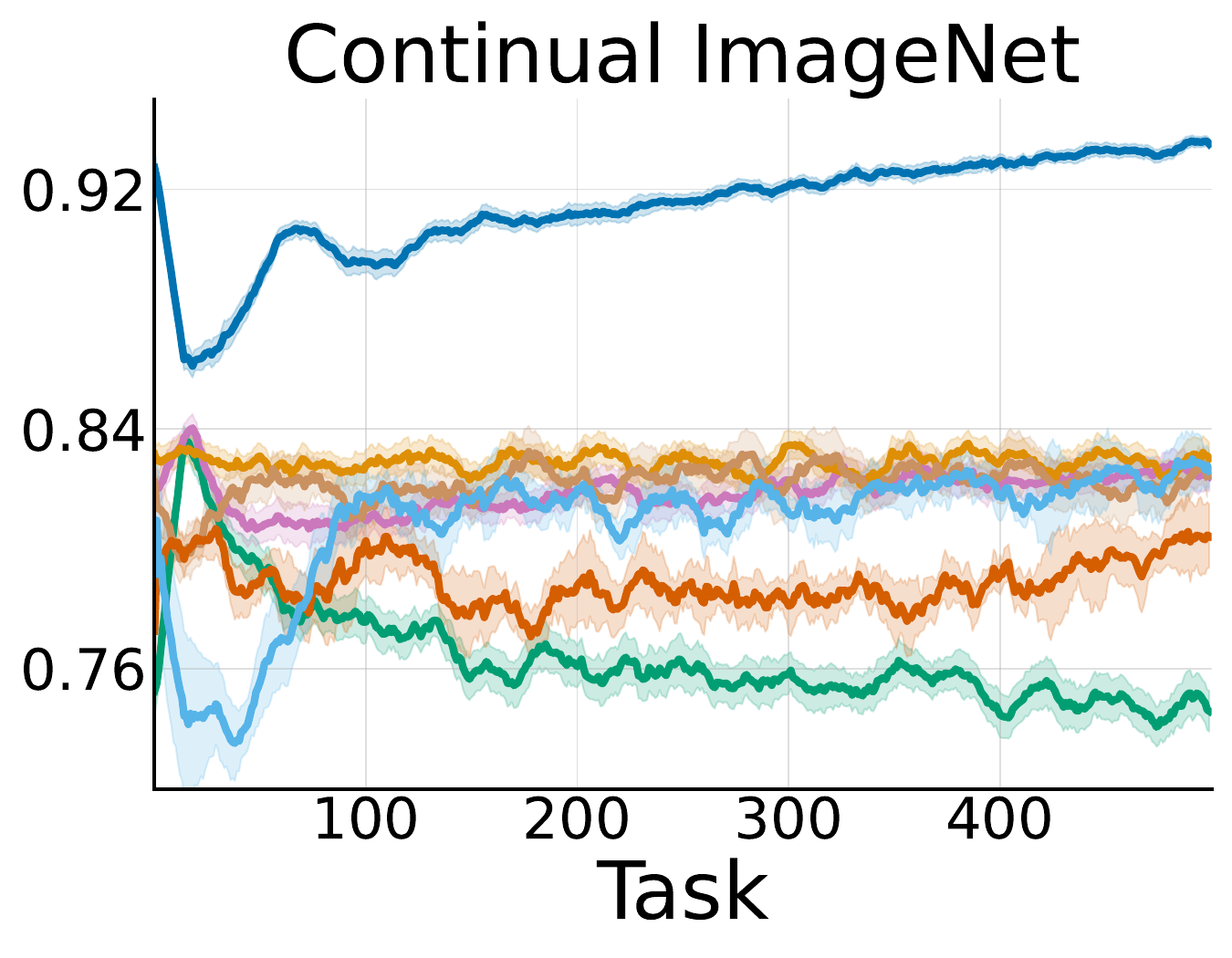}
  \end{subfigure}

    \begin{tikzpicture}

    \def\firstRowY{0.5}
    \def\secondRowY{-0.1}
    
    \draw [LayerNorm, thick, line width=2pt] (-5,\firstRowY) -- (-4.5,\firstRowY);
    \node[anchor=west] at (-4.5,\firstRowY) {LayerNorm};

    \draw [Resetting, thick, line width=2pt] (-2.3,\firstRowY) -- (-1.8,\firstRowY);
    \node[anchor=west] at (-1.8,\firstRowY) {Resetting};

    \draw [NaP, thick, line width=2pt] (0.3,\firstRowY) -- (0.8,\firstRowY);
    \node[anchor=west] at (0.8,\firstRowY) {NaP};

    \draw [spectral, thick, line width=2pt] (2,\firstRowY) -- (2.5,\firstRowY);
    \node[anchor=west] at (2.5,\firstRowY) {Spectral Reg.};

    \draw [ourMethod, thick, line width=2pt] (4.9,\firstRowY) -- (5.4,\firstRowY);
    \node[anchor=west] at (5.4,\firstRowY) {\spectralclip};

        \draw [snp, thick, line width=2pt] (-1.8,\secondRowY) -- (-1.3,\secondRowY);
    \node[anchor=west] at (-1.3,\secondRowY) {Shrink \& Perturb};

    \draw [das, thick, line width=2pt] (1.8,\secondRowY) -- (2.3,\secondRowY);
    \node[anchor=west] at (2.3,\secondRowY) {DASH};

\end{tikzpicture}

  \caption{Comparison of average online accuracy (\Cref{eq:online-acc}) on continual supervised learning tasks. Shaded regions represent bootstrapped estimates of 95\% confidence intervals. Each experiment is repeated over 30 independent runs.}
  \label{fig:cl}
\end{figure}

%% file: appendix-new.tex
\appendix
\onecolumn

\crefalias{section}{appendix}

\renewcommand{\thetheorem}{\thesection.\arabic{theorem}}
\renewcommand{\theproposition}{\thesection.\arabic{proposition}}
\renewcommand{\thelemma}{\thesection.\arabic{lemma}}
\renewcommand{\thecorollary}{\thesection.\arabic{corollary}}
\renewcommand{\thedefinition}{\thesection.\arabic{definition}}
\renewcommand{\theassumption}{\thesection.\arabic{assumption}}

\renewcommand{\theHtheorem}{\thesection.\arabic{theorem}}
\renewcommand{\theHproposition}{\thesection.\arabic{proposition}}
\renewcommand{\theHlemma}{\thesection.\arabic{lemma}}
\renewcommand{\theHcorollary}{\thesection.\arabic{corollary}}
\renewcommand{\theHdefinition}{\thesection.\arabic{definition}}
\renewcommand{\theHassumption}{\thesection.\arabic{assumption}}

\section*{\centering APPENDICES}

For convenience, we collect the contents of the appendix:
\begin{itemize}
 \item \Cref{sec:exp-details} contains experimental details in order to replicate all experiments in the text.
 \item \Cref{sec:algo-details-app} contains details regarding the implementation of \spectralclip and hyperparameter ablations.
 \item \Cref{sec:theory-appendix} restates main theoretical results with proofs.
 \item \Cref{sec:singularclip-proofs} provides proof for theoretical statements relating to \spectralclip.
\end{itemize}

\section{Experimental details}\label{sec:exp-details}

Code for all experiments is provided in supplementary material, and can be used to reproduce the results exactly. We provide details of all empirical setups below.

\textbf{Permuted MNIST. }
We subsample a collection of 10,000 MNIST train dataset images for each experiment run. We trained each model for 10 epochs on each task using a batch size of 500 (so 200 gradient steps per task). We use a MLP comprised of a linear layer of shape $[784, 256]$, a layer normalization layer, a ReLU nonlinearity, and a linear layer of shape $[256, 10]$.

\textbf{Random Label MNIST. }
We subsample a collection of 2,048 MNIST train images for each experiment. We train each model for 10 epochs using a batch size of 512 (so 40 gradient steps per task). We use a MLP comprised of a linear layer of shape $[784, 256]$, a layer normalization layer, a ReLU nonlinearity, and a linear layer of shape $[256, 10]$.

\textbf{Random Label CIFAR. }
We subsample a collection of 2,048 CIFAR-10 train images for each experiment. We train each model for 10 epochs using a batch size of 512 (so 40 gradient steps per task). We use a neural network beginning with two convolutional layers, each followed by a batch normalization layer, ReLU nonlinearity, and a max pooling layer. The output is then flattened and passed to a linear layer of shape $[4096, 256]$, a layer normalization, ReLU nonlinearity, and a linear layer of shape $[256, 10]$.

\textbf{Random Label ImageNet. }
We subsample a collection of 16,384 ImageNet train images for each experiment. We train each model for 20 epochs using a batch size of 4,096 for a total of 80 gradient steps per task. We use a neural network beginning with three convolutional layers, each followed by a batch normalization layer, ReLU nonlinearity, and a max pooling layer. The output is then flattened and passed to a linear layer of shape $[4096, 512]$, a layer normalization, ReLU nonlinearity, and a linear layer of shape $[512, 1000]$.

\textbf{Continual ImageNet. }
For each experiment, we choose the sequence of tasks by fixing a uniform permutation of length 1000 and generating 500 pairs by grouping consecutive elements in the permutation. For each task, we sample 600 samples from each class, and shuffle them for a task dataset of size 1,200. The network is then trained for 10 epochs with a batch size of 100. We use a neural network beginning with three convolutional layers, each followed by a batch normalization layer, ReLU nonlinearity, and a max pooling layer. The output is then flattened and passed to a linear layer of shape $[4096, 512]$, a layer normalization, ReLU nonlinearity, and a linear layer of shape $[512, 1000]$.

\section{Algorithmic Details}\label{sec:algo-details-app}
\subsection{Hyperparameter Ablations}\label{sec:hparam-sweeps}

\begin{figure}[h]
 \centering
 \includegraphics[width=\linewidth]{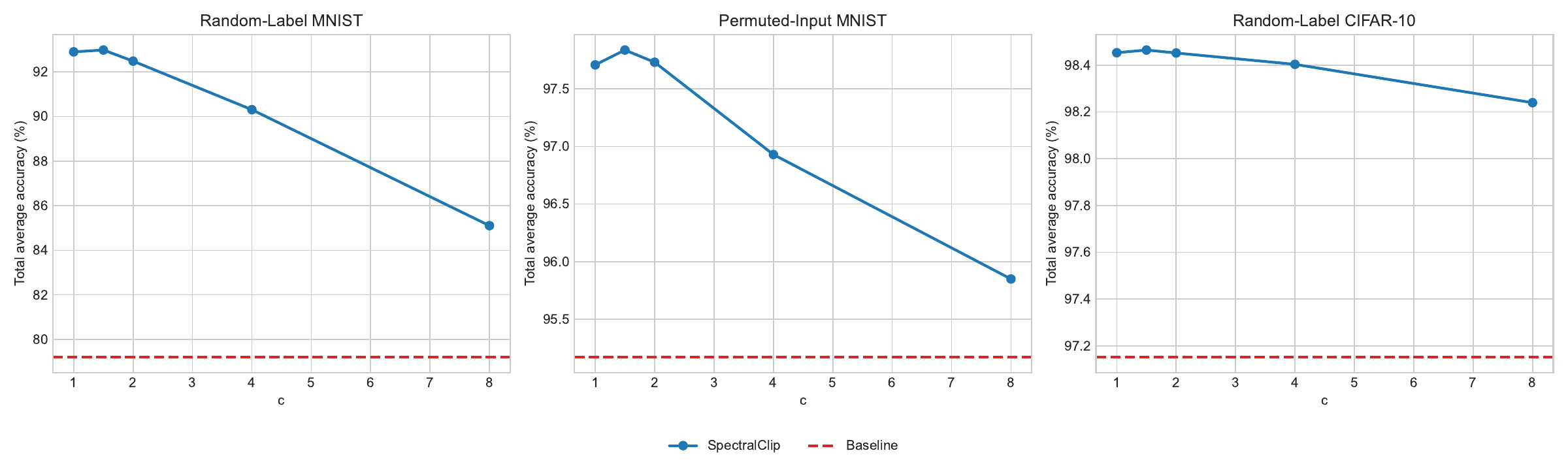}
 \caption{Ablation of $c$ over a subset of continual learning tasks.}
 \label{fig:cl-sweep}
\end{figure}

In \Cref{fig:cl-sweep} we perform an ablation of different values of $c$ across continual learning tasks, and similarly for RL tasks in \Cref{fig:rl-sweep}. We generally find that continual learning tasks are optimized at lower values of $c$, this reflects a larger degree of nonstationarity in the continual learning tasks. We compare the performance relative to a LayerNorm baseline.

A larger difference in behaviour can be observed in the two RL tasks: in the UTD-1 setting, the optimal $c$ is not found, and performance seems to continuously increase with $c$. This is reflective of our results in \Cref{sec:empirical}, where \spectralclip underperformed SAC. This reflects our \emph{undertuning} of our method by not tuning $c$ per-UTD (which, as stated we did to obtain a fair comparison), as a properly tuned $c$ should never underperform the baseline on a task (since as $c\to\infty$ we recover the baseline algorithm), in the UTD-8 setting, we can see that the optimal $c$ is reached at $c=4$, demonstrating the higher degree of plasticity loss in this setting.

In \Cref{fig:one-sided-sweep}, we present an ablation of one-sided clipping across continual learning tasks. As suggested in \Cref{sec:singularclip}, when only lower-clipping is performed, the plasticity loss is dominated by the effective learning rate decay due to parameter norm growth, and its performance is similar to the baseline algorithm across all $c$. Upper-clipping alone achieves strong gains compared to the baseline, however it still underperforms the two-sided clipping results present in \Cref{fig:cl-sweep}, as it is not able to prevent AIPL since the lower singular values are unbounded.

\begin{figure}[h]
 \centering
 \includegraphics[width=\linewidth]{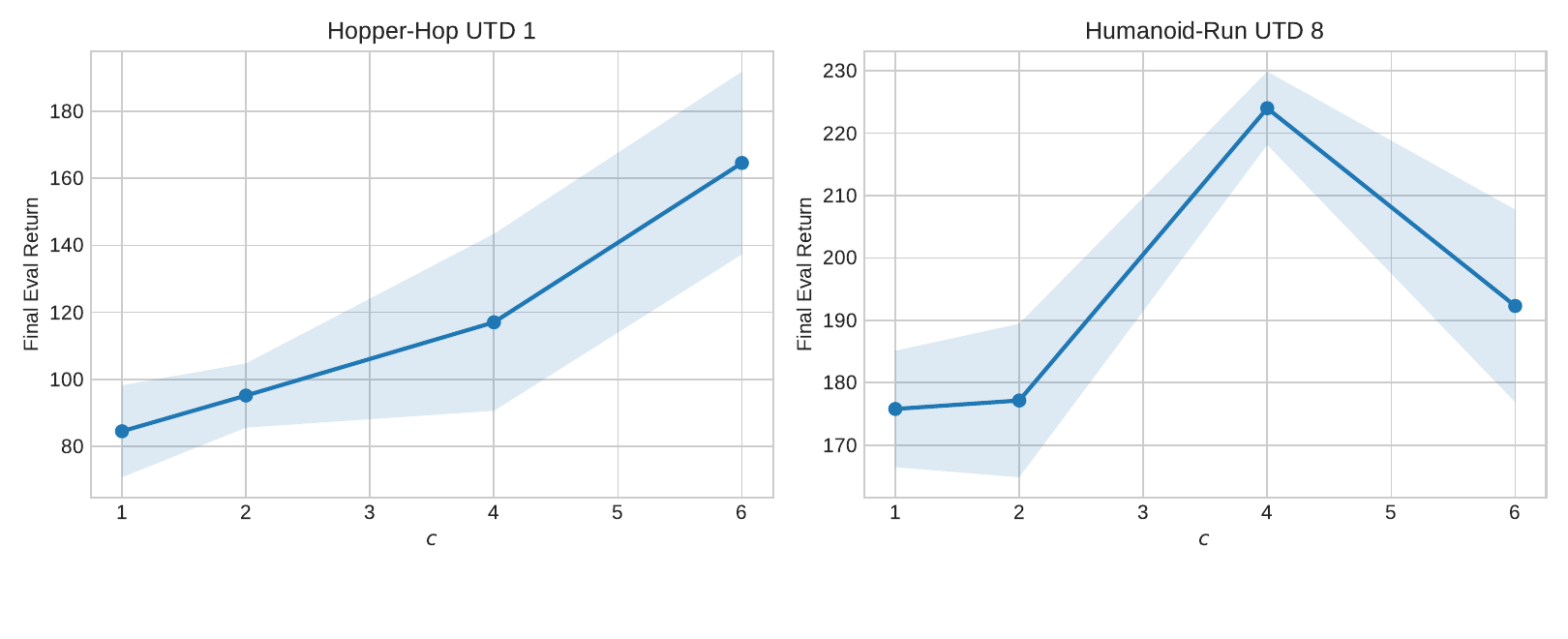}
 \caption{Ablation of $c$ over a subset of RL settings.}
 \label{fig:rl-sweep}
\end{figure}

\begin{figure}[h]
 \centering
 \includegraphics[width=\linewidth]{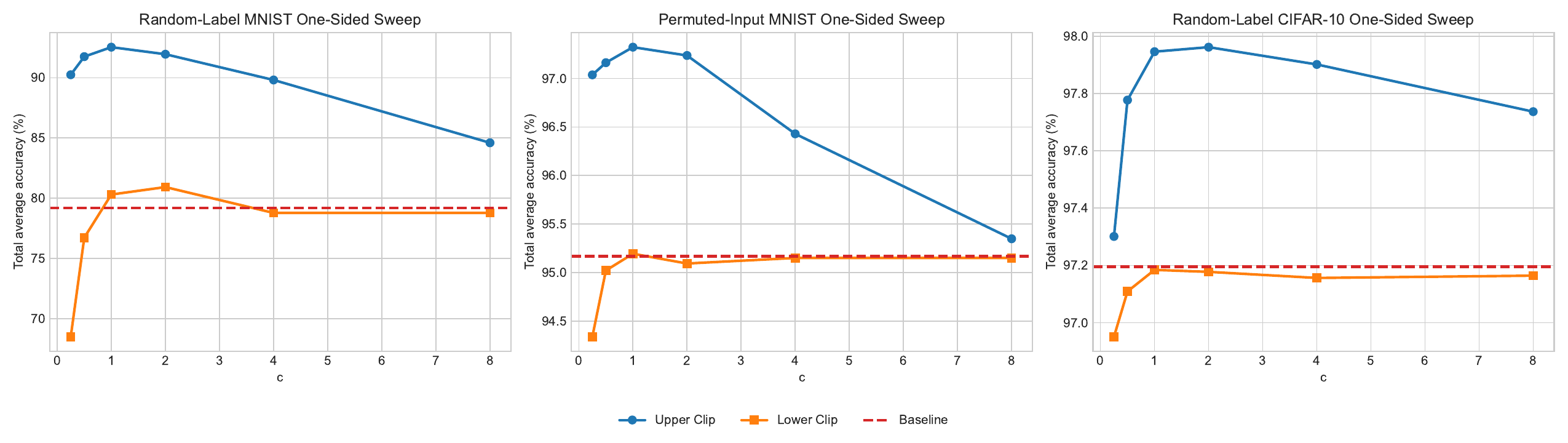}
 \caption{Ablation of one-sided clipping across a subset of continual learning tasks. The baseline curve corresponds to a network with no intervention applied.}
 \label{fig:one-sided-sweep}
\end{figure}

\begin{figure}[h]
 \centering
 \includegraphics[width=1\linewidth]{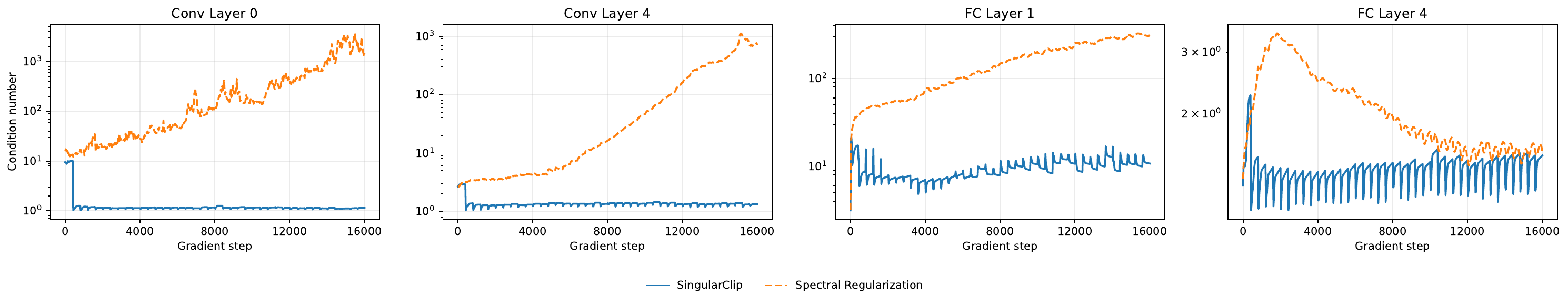}
 \caption{Comparison of condition number growth between SingularClip and spectral regularization on the Random Label CIFAR task, across different network layers.}
\end{figure}

\subsection{Implementation of Clipping}\label{sec:clipping-impl}

We now highlight two possible methods for implementing the clipping procedure $\mathrm{sc}^b_a$.

The first method is straightforward: we directly compute the SVD of each weight matrix, clip the singular value matrix, and multiply the remaining matrices. This results in accurate numerical results, but is unable to fully take advantage of tensor cores and features of modern GPUs, and can be unstable in mixed-precision settings.

The second method is based on Newton-Schulz iterations, odd matrix polynomials with certain coefficients designed to approximate matrix orthogonalization, which are used in the Muon optimizer \citep{jordan2024muon}. Recent works have introduced different polynomial coefficients in order to converge faster and with less error \citep{cesista2025muonoptcoeffs, amsel2025polar}. These orthogonalization procedures can then be used to perform singular value clipping as described by \citet{cesista2025spectralclipping}.

In our evaluations, we found that the Newton-Schulz-based singular clipping achieved lower performance than SVD-based, due to the precision error incurred. We produced estimates of the Frobenius error when clipping to $[1/2, 2]$ in \Cref{tab:orthogonalization_error}. Furthermore, we found that since we were not applying singular clipping frequently (e.g., every 200,000 gradient steps in the RL experiments in \Cref{sec:empirical}), the overhead incurred from computing the SVD was insignificant. We display the measured wall clock times for various matrix shapes in \Cref{tab:clip_runtime}, and note that in the worst case considered the increase is on the order of 50ms, which is negligible when the clipping is done infrequently. 

While we used SVD for our experiments, we note that in larger-scale architectures where SVD may become a bottleneck, Newton-Schulz iterations can be used as an efficient alternative. As a note on its scalability, \citet{newhouse2025training} applied Newton-Schulz-based singular clipping at every gradient step for up to 145M parameter transformer models, demonstrating that \spectralclip can be scaled to similar settings without significant overhead.

\begin{table}[ht]
\centering
\caption{Relative Frobenius error of Newton-Schulz clipping for $c=2$ across different matrix shapes. Each value is computed by averaging over 30 runs. Uncertainties indicate two standard errors.}

\begin{tabular}{ll}
\hline
\textbf{Shape} &
Relative Frobenius Error (\%) \\
\hline
(128, 128) & $8.850 \times 10^{-3}$ \scriptsize{\textcolor{gray}{$\pm$ $1.9 \times 10^{-3}$}} \\
(256, 256) & $1.019 \times 10^{-2}$ \scriptsize{\textcolor{gray}{$\pm$ $8.8 \times 10^{-4}$}} \\
(512, 512) & $1.118 \times 10^{-2}$ \scriptsize{\textcolor{gray}{$\pm$ $3.6 \times 10^{-4}$}} \\
(1024, 1024) & $1.076 \times 10^{-2}$ \scriptsize{\textcolor{gray}{$\pm$ $1.3 \times 10^{-4}$}} \\
(256, 512) & $7.119 \times 10^{-3}$ \scriptsize{\textcolor{gray}{$\pm$ $9.9 \times 10^{-5}$}} \\
\hline
\end{tabular}
\label{tab:orthogonalization_error}
\end{table}

\begin{table}[ht]
\centering
\caption{Wallclock time comparisons of singular clipping methods on a NVIDIA L40S GPU. Each value is computed by averaging over 30 runs after 10 warmup iterations. Uncertainties indicate two standard errors.}
\begin{tabular}{lcc}
\hline
\textbf{Shape} &
SVD (ms) &
Newton-Schulz (ms) \\
\hline
(128, 128) & 2.527 \scriptsize{\textcolor{gray}{$\pm$ 0.075}} & 2.219 \scriptsize{\textcolor{gray}{$\pm$ 0.013}} \\
(256, 256) & 6.282 \scriptsize{\textcolor{gray}{$\pm$ 0.006}} & 2.540 \scriptsize{\textcolor{gray}{$\pm$ 0.653}} \\
(512, 512) & 17.390 \scriptsize{\textcolor{gray}{$\pm$ 0.017}} & 2.287 \scriptsize{\textcolor{gray}{$\pm$ 0.008}} \\
(1024, 1024) & 51.230 \scriptsize{\textcolor{gray}{$\pm$ 0.259}} & 4.887 \scriptsize{\textcolor{gray}{$\pm$ 0.029}} \\
(256, 512) & 6.515 \scriptsize{\textcolor{gray}{$\pm$ 0.009}} & 2.417 \scriptsize{\textcolor{gray}{$\pm$ 0.022}} \\
\hline
\end{tabular}
\label{tab:clip_runtime}
\end{table}

\subsection{\spectralclip with Muon baseline}

One may wonder whether AIPL is an issue when using Muon \citep{jordan2024muon} as an optimizer: indeed, Muon orthogonalizes the update before adding it to the weight matrices, and with perfect orthogonalization, this would not lead to anisotropy. However, since Muon employs Newton-Schulz iterations \citep{jordan2024muon, cesista2025muonoptcoeffs, amsel2025polar}, we find that the numerical error accrued leads to an increase in anisotropy, and hence AIPL. We present this on the Continual ImageNet task in \Cref{fig:muon-comparison}. In the left panel, we can see that Muon dampens the anisotropy growth observed with Adam as seen in \Cref{fig:cifar-layer-conds}, but still experience a growth nonetheless. This indeed leads to a plasticity loss in the right panel, as the Muon baseline is not able to achieve the same performance as Muon with SingularClip.

\begin{figure}[h]
 \centering
 \includegraphics[width=\linewidth]{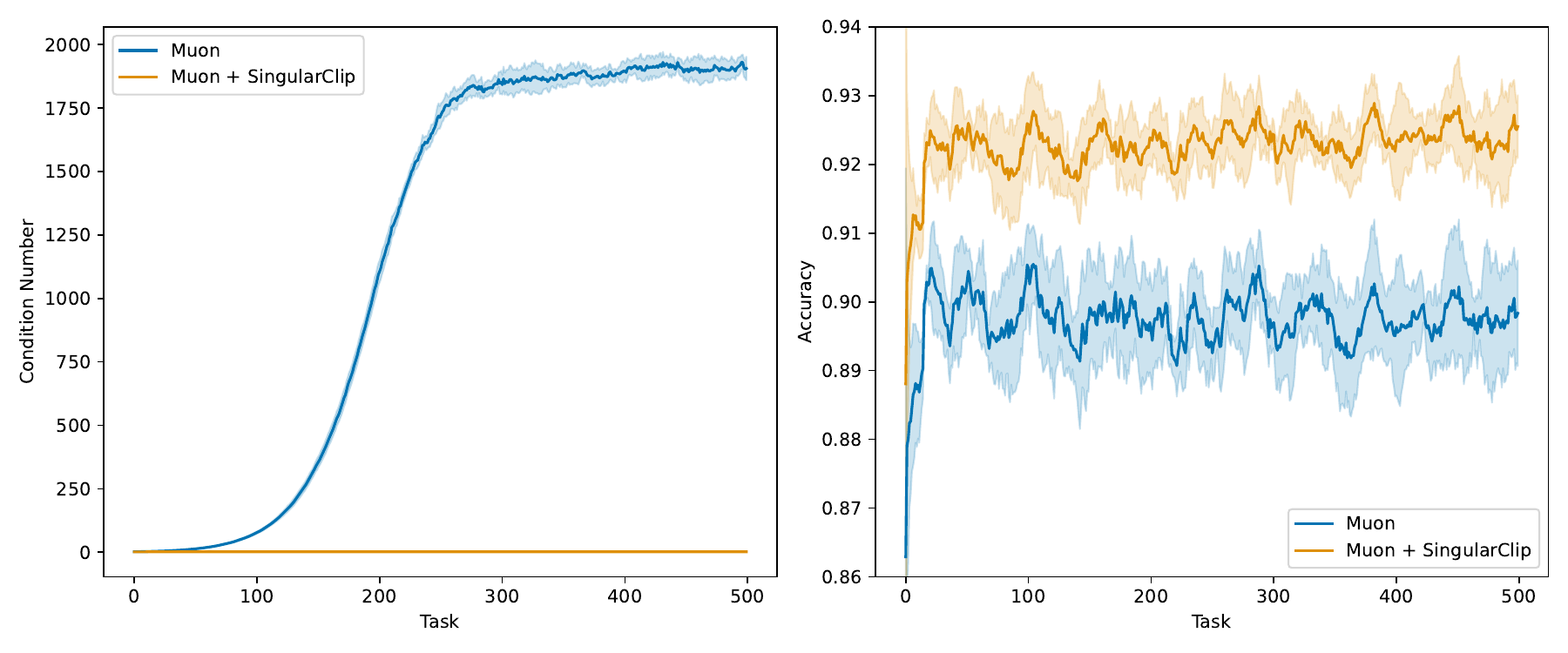}
 \caption{Comparison of adding SingularClip to Muon on the Continual ImageNet task. We plot average weight condition number (left) and online accuracy (right). Error bars represent bootstrapped 95\% confidence intervals across 10 runs.}
 \label{fig:muon-comparison}
\end{figure}

\section{Theoretical Results}\label{sec:theory-appendix}
\addcontentsline{toc}{section}{Theoretical Appendix}
\setcounter{theorem}{0}

We now introduce the theoretical setting and prove a number of auxiliary results to prove the main theorem. The statements and proofs of the main results can be found in Appendix~\ref{app:weak-dynamics}.

\textbf{Model, Data, and Loss.} We consider a two-layer linear network with input dimension $d$, hidden width $r$, and $K$ output classes. The trainable weights are $W_1 \in \mathbb{R}^{r \times d}$ and $W_2 \in \mathbb{R}^{K \times r}$. The network outputs are $f(x;W)=W_{2}W_{1}x$, yielding probabilities $p(x;W)=\text{Softmax}(f(x;W))$. We assume a realizable setting provided by a target distribution $p^*(x)$.

We denote random variables by uppercase letters (e.g., $X$) and realizations by lowercase (e.g., $x$). Vectors are columns; the Euclidean norm is $\|\cdot\|$. For a matrix $A$, $\|A\|_F$ is the Frobenius norm and $\langle A,B\rangle_F = \text{Tr}(A^{\top}B)$ denotes the Frobenius inner product. For compatible matrices, $\|Ax\|_{2} \le \|A\|_{op}\|x\|_{2} \le \|A\|_{F}\|x\|_{2}$.

\textbf{Parameter Space Geometry.} We treat the network parameters as a tuple $W = (W_1, W_2)$ residing in the product space $\mathcal{P} = \mathbb{R}^{r \times d} \times \mathbb{R}^{K \times r}$. We equip $\mathcal{P}$ with the standard Euclidean product norm:

$$ \|W\|_{\mathcal{P}} := \sqrt{\|W_1\|_F^2 + \|W_2\|_F^2}. $$

The associated inner product is $\langle A, B \rangle_{\mathcal{P}} := \langle A_1, B_1 \rangle_F + \langle A_2, B_2 \rangle_F$.

\textbf{Complementary Projections.} Let $\mathcal{S} \subseteq \mathbb{R}^r$ be a subspace and $P_{\mathcal{S}} \in \mathbb{R}^{r \times r}$ be the orthogonal projector onto $\mathcal{S}$ (satisfying $P_{\mathcal{S}}^2 = P_{\mathcal{S}}$ and $P_{\mathcal{S}}^\top = P_{\mathcal{S}}$). We define the complementary projector onto the weak subspace as $P_{\mathcal{W}} := I_r - P_{\mathcal{S}}$, which projects onto the orthogonal complement $\mathcal{W} = \mathcal{S}^\perp$.

\begin{proposition}[Pythagorean identities]\label{prop:pythagorean}

Let $\mathcal{S} \subseteq \mathbb{R}^r$ be a subspace with orthogonal projector $P_{\mathcal S}$ and complementary projector $P_{\mathcal W} = I_r - P_{\mathcal S}$.

The Frobenius norm decomposes additively under these projections:

\begin{enumerate}

\item \textbf{Left-Projection:} For any $A \in \mathbb{R}^{r \times n}$, $\|A\|_F^2 = \|P_{\mathcal S}A\|_F^2 + \|P_{\mathcal W}A\|_F^2$.

\item \textbf{Right-Projection:} For any $B \in \mathbb{R}^{m \times r}$, $\|B\|_F^2 = \|B P_{\mathcal S}\|_F^2 + \|B P_{\mathcal W}\|_F^2$.

\end{enumerate}

\end{proposition}

\begin{proof}

\textbf{Left-Projection:} We expand the squared norm of $A = (P_{\mathcal S} + P_{\mathcal W})A$:
$$ \|A\|_F^2 = \langle P_{\mathcal S}A + P_{\mathcal W}A, P_{\mathcal S}A + P_{\mathcal W}A \rangle_F = \|P_{\mathcal S}A\|_F^2 + \|P_{\mathcal W}A\|_F^2 + 2\langle P_{\mathcal S}A, P_{\mathcal W}A \rangle_F. $$

The cross term vanishes due to orthogonality. Using $\langle X, Y \rangle_F = \operatorname{Tr}(X^\top Y)$:
$$ \langle P_{\mathcal S}A, P_{\mathcal W}A \rangle_F = \operatorname{Tr}(A^\top P_{\mathcal S}^\top P_{\mathcal W} A) = \operatorname{Tr}(A^\top (P_{\mathcal S} P_{\mathcal W}) A) = 0, $$

since $P_{\mathcal S} P_{\mathcal W} = 0$.

\textbf{Right-Projection:} Similarly, for $B = B(P_{\mathcal S} + P_{\mathcal W})$:
$$ \|B\|_F^2 = \|B P_{\mathcal S}\|_F^2 + \|B P_{\mathcal W}\|_F^2 + 2\operatorname{Tr}((B P_{\mathcal S})^\top B P_{\mathcal W}). $$
The trace term is $\operatorname{Tr}(P_{\mathcal S}^\top B^\top B P_{\mathcal W})$. By the cyclic property of the trace, this equals $\operatorname{Tr}(P_{\mathcal W} P_{\mathcal S} B^\top B) = 0$.
\end{proof}

Before analyzing the network's internal dynamics, we must specify the external environment it interacts with by formalizing the data distribution and the learning objective.

\begin{assumption}[Data]\label{assump:data}
The input is a random variable $X \in \mathbb{R}^d$ drawn from a distribution with $\mathbb{E}[\|X\|^2] < \infty$ and covariance $\Sigma = \mathbb{E}[XX^{\top}] \succ 0$. We assume the target labels are generated by a teacher model $p^*$.
\end{assumption}

\textbf{Loss and Gradient Flow.} We denote the cross-entropy for two probability vectors $p, q$ as $H(p, q) = -\sum_{k=1}^K p_k \log q_k$.

The population cross-entropy loss is defined as:
\[ L(W) := \mathbb{E}_x[H(p^*(x), p(x;W))]. %
\]
We consider the weights trained using projected gradient flow, which admits the dynamics
\begin{align}
 \dot{W}_1(t) &=-\nabla_{W_1}L(W(t)) + \frac{\langle \nabla_{W_1}L(W(t)),W_1(t)\rangle_F}{\|W_1(t)\|_F^2}W_1(t), \\
 \dot{W}_2(t) &=-\nabla_{W_2}L(W(t)) + \frac{\langle \nabla_{W_2}L(W(t)),W_2(t)\rangle_F}{\|W_2(t)\|_F^2}W_2(t).
\end{align}

We now prove that if we assume that $\|W_1(0)\|_F=\|W_2(0)\|_F=1$, then we have the invariance $\|W_1(t)\|_F=\|W_2(t)\|_F=1$.

\begin{lemma}[Layer-wise norm preservation]\label{lem:norm-preservation-projected}
Along the projected gradient flow,
\[
\frac{d}{dt}\|W_1(t)\|_F^2 = 0,
\qquad
\frac{d}{dt}\|W_2(t)\|_F^2 = 0.
\]
Hence $\|W_1(t)\|_F\equiv \|W_1(0)\|_F$ and $\|W_2(t)\|_F\equiv \|W_2(0)\|_F$. In particular, after rescaling the initialization we may assume without loss of generality that
\[
\|W_1(t)\|_F = \|W_2(t)\|_F = 1
\qquad
\text{for all } t\ge 0.
\]
\end{lemma}

\begin{proof}
For $W_1$,
\begin{align*}
\frac{d}{dt}\|W_1\|_F^2
&= 2\langle W_1,\dot W_1(t)\rangle_F\\
&= -2\left\langle W_1(t), \nabla_{W_1}L(W(t)) - \frac{\langle \nabla_{W_1}L(W(t)),W_1(t)\rangle_F}{\|W_1(t)\|_F^2}W_1(t) \right\rangle_F\\
&= 0.
\end{align*}
The proof for $W_2$ is identical.
\end{proof}

This projected gradient flow can be interpreted as the gradient flow when activation normalization is present after each of the layers, in particular, the map 
\[ \mathbb{R}^d\to\mathbb{R}^d\ni \mathrm{norm}(x): x\mapsto \frac{x - \frac1d\sum_{i=1}^dx_i}{(\frac1d\sum_{i=1}^d(x_i- \frac1d\sum_{j=1}^dx_j)^2)^\frac12}. \]
To see this, note that this map has the effect of setting the gradient orthogonal to the current parameter vector, as it makes the function scale-invariant. This can be seen as a layer normalization layer \citep{ba2016layer} without learnable affine parameters.

We will establish a uniform geometric bound on the prediction error. This ensures that the `driving force' of the optimization remains finite, regardless of the network state. Let us define $\Delta p(x;W)\coloneqq p(x;W)-p^*(x)$.

\begin{lemma}\label{lem:softmax-residual}
For all $x$, $\|\Delta p(x;W)\|_2\le \sqrt{2}$.
\end{lemma}
\begin{proof}
Let $p, p^* \in \Delta^{K-1}$. Then $\| p-p^* \|_2^2 = \|p\|_2^2 + \|p^*\|_2^2 - 2\langle p, p^* \rangle$. Since probabilities are non-negative, $\langle p, p^* \rangle \ge 0$. Since $p$ is on the simplex, $\|p\|_2 \le \|p\|_1 = 1$. Thus $\|\Delta p\|_2^2 \le 1 + 1 = 2$.
\end{proof}

We next verify the smoothness properties of the loss function. This ensures that the projected gradient flow is well-defined and the curvature remains bounded, preventing finite-time singularities driven purely by the loss geometry.

\begin{lemma}[Regularity and Hessian bound]\label{lem:diff-under-exp} Under \Cref{assump:data}, the population loss $L(W)$ is $C^2$. For the sample loss with logits $f=W_2W_1x$, the Hessian of the loss with respect to weights, considered as a bilinear form on directions $U=(U_1, U_2)$ with unit norm $\|U\|_{\mathcal{P}}=1$, satisfies:
\[
\nabla^2_W l(x;W)[U,U] \le \|W\|^2_{\mathcal{P}} \|x\|^2 + \sqrt{2} \|x\|.
\]
\end{lemma}

\begin{proof} 
\textbf{1. Hessian of the Softmax.} Let $l(f)$ be the cross-entropy loss w.r.t. logits $f\in\mathbb{R}^K$. The gradient is $\nabla_f l = p - p^*$. The Hessian $H_f = \nabla^2_f l$ is the Jacobian of the softmax map. Explicit calculation yields $[H_f]_{ij} = p_i \delta_{ij} - p_i p_j$. In matrix form, $H_f = \operatorname{diag}(p) - pp^\top$. For any vector $v \in \mathbb{R}^K$, the quadratic form is $v^\top H_f v = \sum_{k=1}^K p_k v_k^2 - (\sum_{k=1}^K p_k v_k)^2$.

This expression corresponds exactly to the variance of a discrete random variable taking the value $v_k$ with probability $p_k$. Because variance is always non-negative, the quadratic form satisfies $v^\top H_f v \ge 0$. Furthermore, the variance is upper-bounded by the uncentered second moment, yielding $v^\top H_f v \le \sum_{k=1}^K p_k v_k^2$. Because the probabilities reside on the simplex ($p_k \le 1$), we unconditionally bound this convex combination by the unweighted sum of squares: $\sum_{k=1}^K p_k v_k^2 \le \sum_{k=1}^K v_k^2 = \|v\|_2^2$. Thus, the quadratic form satisfies $0 \le v^\top H_f v \le \|v\|_2^2$. Because $H_f$ is a symmetric positive semi-definite covariance matrix, its spectral operator norm $\|H_f\|_{\mathrm{op}}$ exactly equals its maximum eigenvalue. By the Rayleigh quotient, the maximum eigenvalue is exactly $\sup_{v \neq 0} \frac{v^\top H_f v}{\|v\|_2^2}$, which is rigorously bounded by $1$. This sequence of inequalities guarantees $\|H_f\|_{\mathrm{op}} \le 1$.

\textbf{2. Hessian w.r.t. Weights.} We apply the second-order chain rule to the composition $l(f(W))$. For a direction $U=(U_1, U_2)$, the second derivative is:
\[
\nabla^2_W l[U,U] = (\nabla_W f[U])^\top H_f (\nabla_W f[U]) + (\nabla_f l)^\top (\nabla^2_W f[U,U]).
\]
(i) Gauss-Newton Term: The first variation is $\nabla_W f[U] = (W_2 U_1 + U_2 W_1)x$. By the triangle inequality:
\[ \|\nabla_W f[U]\| \le \|W_2 U_1 x\| + \|U_2 W_1 x\| \le (\|W_2\|_F \|U_1\|_F + \|W_1\|_F \|U_2\|_F)\|x\|. \]
We apply the Cauchy-Schwarz inequality to the vector pairs $(\|W_2\|_F, \|W_1\|_F)$ and $(\|U_1\|_F, \|U_2\|_F)$:
\[ \|W_2\|_F \|U_1\|_F + \|W_1\|_F \|U_2\|_F \le \sqrt{\|W_1\|_F^2 + \|W_2\|_F^2} \sqrt{\|U_1\|_F^2 + \|U_2\|_F^2} = \sqrt{2} \|U\|_{\mathcal{P}}. \]
Recall that for the softmax Hessian $H_f$, the quadratic form is bounded by the squared norm of the vector: $(\nabla_W f[U])^\top H_f (\nabla_W f[U]) \le \|\nabla_W f[U]\|^2$. Squaring the Jacobian bound derived above yields:
$$ \|\nabla_W f[U]\|^2 \le \left(\|W\|_{\mathcal{P}} \|U\|_{\mathcal{P}} \|x\|\right)^2 = 2 \|x\|^2, $$
where we used the normalization $\|U\|_{\mathcal{P}}=1$.

(ii) Newton Term: The map $f(W) = W_2 W_1 x$ is bilinear. To find the second directional derivative along $U = (U_1, U_2)$, we evaluate $f(W + tU)$ with respect to the scalar $t$:
$$ f(W + tU) = (W_2 + tU_2)(W_1 + tU_1)x = W_2W_1x + t(W_2U_1 + U_2W_1)x + t^2 U_2U_1x. $$

The second-order term in this polynomial expansion corresponds to $\frac{1}{2} \nabla^2_W f[U,U] t^2$. Equating the $t^2$ coefficients yields the exact second derivative: $\nabla^2_W f[U,U] = 2U_2 U_1 x$.

Its norm is bounded using sub-multiplicativity by $2\|U_2\|_F \|U_1\|_F \|x\|$. We restrict $U$ to the unit sphere ($\|U_1\|^2_F + \|U_2\|^2_F = 1$). By the AM-GM inequality ($2ab \le a^2 + b^2$), we have:
$$ 2\|U_1\|_F \|U_2\|_F \le \|U_1\|_F^2 + \|U_2\|_F^2 = 1. $$

Thus $\|\nabla^2_W f[U,U]\| \le \|x\|$. Finally, $\|\nabla_f l\| = \|p - p^*\| \le \sqrt{2}$ (\Cref{lem:softmax-residual}). Combining these yields the bound.

\textbf{3. Regularity.} The Hessian components are dominated by the integrable envelope functions $\|x\|^2$ and $\|x\|$. By \Cref{assump:data}, $\mathbb{E}[\|X\|^2] < \infty$, permitting differentiation under the expectation. Thus $L(W)$ is $C^2$. 
\end{proof}

\medskip \par \noindent \textbf{Remark.} This lemma bounds the curvature of the optimization objective. The first term, scaling quadratically with the input energy $\|x\|^2$, represents the Hessian scale with input magnitude. The second term, proportional to $\sqrt{2}\|x\|$, represents the baseline curvature originating from the raw input scale and the worst-case probability error. Together, these terms guarantee that the curvature remains deterministic and bounded, preventing optimization divergence.

\subsection{Initialization and Singular-Subspace Mass}\label{app:init-singular-mass}

We consider an arbitrary hidden subspace $\mathcal U\subseteq\mathbb R^r$ and track the amount of parameter mass assigned to it.

\begin{definition}[Singular-subspace mass]\label{def:app-joint}
Let $\mathcal U\subseteq\mathbb R^r$ and let $P_{\mathcal U}$ be the orthogonal projector onto $\mathcal U$. Define
\[
M_{\mathcal U}(W) := \|P_{\mathcal U}W_1\|_F^2 + \|W_2P_{\mathcal U}\|_F^2.
\]
\end{definition}

\medskip \par \noindent \textbf{Remark.} The quantity $M_{\mathcal U}(W)$ measures how much hidden-layer mass the network assigns to directions in $\mathcal U$: from the left for $W_1$ and from the right for $W_2$. 

\begin{proposition}\label{prop:initial-singular-mass}
Let $\mathcal U\subseteq\mathbb R^r$ be spanned by left singular vectors of $W_1(0)$ indexed by $J_1$ and right singular vectors of $W_2(0)$ indexed by $J_2$. Then
\begin{align*}
M_{\mathcal U}(W(0))
&\le
\sum_{j\in J_1}\sigma_{1,j}(0)^2 + \sum_{j\in J_2}\sigma_{2,j}(0)^2 \\
&\le
|J_1|\max_{j\in J_1}\{\sigma_{1,j}\} + |J_2|\max_{j\in J_2}\{\sigma_{2,j}\}\\
&\coloneqq |J_1|\sigma_{\mathcal{U}, 1} + |J_2|\sigma_{\mathcal{U}, 2}.
\end{align*}
\end{proposition}

\begin{proof}
Expand $W_1(0)$ and $W_2(0)$ in singular value decompositions. Since $P_{\mathcal U}$ acts as the identity on the selected singular directions and as a contraction on all others, the claim follows from orthogonality of distinct singular vectors and the Pythagorean identities from \Cref{prop:pythagorean}.
\end{proof}

\subsection{Gradient Identities and Trajectory Properties}\label{app:grad-identities}

Let $G_i\coloneqq\nabla_{W_i}L(W)$. We derive the explicit forms of the population gradients to understand how the error signal interacts with the layers.
\begin{lemma}[Population gradients]\label{lem:population-gradients}
$G_1=\mathbb{E}[W_2^\top\,\Delta p\, x^\top]$ and $G_2=\mathbb{E}[\Delta p\,(W_1x)^\top]$.

\end{lemma}
\begin{proof}
We compute the gradients using Fréchet differentials and the Frobenius inner product $\langle A, B \rangle_F = \operatorname{Tr}(A^\top B)$. Let $f = W_2 W_1 x$ be the vector of logits. The differential of the per-sample loss $l$ with respect to the network output is $dl = \langle \nabla_f l, df \rangle_2$, where $\nabla_f l = p(x; W) - p^*(x) =: \Delta p$.

By the multivariable product rule, the differential of the network output with respect to the continuous weight matrices is $df = (dW_2) W_1 x + W_2 (dW_1) x$. Substituting this into the loss differential yields:
$$ dl = \langle \Delta p, (dW_2) W_1 x \rangle_2 + \langle \Delta p, W_2 (dW_1) x \rangle_2. $$

We isolate the differentials by rewriting the Euclidean inner products as traces and utilizing the cyclic property of the trace ($\operatorname{Tr}(ABC) = \operatorname{Tr}(CAB)$):

For $W_2$:
$$ \langle \Delta p, (dW_2) W_1 x \rangle_2 = \operatorname{Tr}(\Delta p^\top (dW_2) W_1 x) = \operatorname{Tr}((W_1 x) \Delta p^\top dW_2) = \langle \Delta p (W_1 x)^\top, dW_2 \rangle_F. $$
This extracts the exact per-sample gradient as $\nabla_{W_2} l = \Delta p (W_1 x)^\top$.

For $W_1$:
$$ \langle \Delta p, W_2 (dW_1) x \rangle_2 = \operatorname{Tr}(\Delta p^\top W_2 (dW_1) x) = \operatorname{Tr}(x \Delta p^\top W_2 dW_1) = \langle W_2^\top \Delta p \, x^\top, dW_1 \rangle_F. $$
This extracts the per-sample gradient as $\nabla_{W_1} l = W_2^\top \Delta p \, x^\top$.

Taking expectations over the data distribution $x$ yields the population gradients $G_1 = \mathbb{E}[W_2^\top \Delta p \, x^\top]$ and $G_2 = \mathbb{E}[\Delta p (W_1 x)^\top]$.
\end{proof}

\medskip \par \noindent \textbf{Remark.} This lemma highlights the pathways through which the environment's error signal updates the network parameters. By extracting the explicit algebraic forms $W_2^\top \Delta p \, x^\top$ and $\Delta p (W_1 x)^\top$, the result highlights an asymmetric structural dependency: the gradient updates to the first layer are scaled by the magnitude of the second layer's weights. This reveals a gradient attenuation bottleneck discussed in \Cref{sec:aipl}: if the weights in $W_2$ undergo spectral collapse, the gradient to $W_1$ is attenuated, hindering the network's ability to extract and adapt to newly relevant input features.

\begin{definition}[Coupling constant]\label{def:coupling}
$A := \sqrt{2} \mathbb{E}\|x\|$.
\end{definition}

\medskip \par \noindent \textbf{Remark.} This constant $A$ encapsulates the interaction strength between the error signal and the data distribution. Specifically, the factor $\sqrt{2}$ bounds the maximum divergence length of the prediction error vector on the probability simplex, while the expectation term $\mathbb{E}\|x\|$ accounts for the energy scale of the incoming data. These combined aspects dictate the maximum leverage a sample can exert on the network's weights, thereby setting the continuous timescale of the ODE dynamics.

\begin{lemma}\label{lem:projected-gradient-U}
For every fixed projector $P\in\mathbb R^{r\times r}$,
\[
PG_1 = \mathbb E\big[(W_2P)^\top\Delta p\,x^\top\big],
\qquad
G_2P = \mathbb E\big[\Delta p\,(PW_1x)^\top\big].
\]
Consequently,
\[
\|PG_1\|_F \le A\,\|W_2P\|_F,
\qquad
\|G_2P\|_F \le A\,\|PW_1\|_F.
\]
\end{lemma}

\begin{proof}
The identities follow by projecting the formulas in \Cref{lem:population-gradients} and using $P=P^\top$. For instance,
\[
PG_1 = P\,\mathbb E[W_2^\top\Delta p\,x^\top]
= \mathbb E[PW_2^\top\Delta p\,x^\top]
= \mathbb E[(W_2P)^\top\Delta p\,x^\top].
\]
The second identity is analogous. For the norm bounds, use Jensen's inequality, $\|uv^\top\|_F = \|u\|_2\|v\|_2$, the operator-to-Frobenius inequality, and \Cref{lem:softmax-residual}.
\end{proof}

\subsection{Subspace Dynamics and Exponential Envelope}\label{app:weak-dynamics}

We now determine the instantaneous growth rate of the weak energy. We show that the time derivative of the energy is upper-bounded by the energy itself, setting the stage for exponential dynamics.

\begin{proposition}\label{thm:weak-growth}
For every fixed hidden subspace $\mathcal U$,
\[
\frac{d}{dt}M_{\mathcal U}(t) \le 4A\,M_{\mathcal U}(t).
\]
\end{proposition}

\begin{proof}
We proceed by differentiating
\[
M_{\mathcal U}(t)=\|PW_1\|_F^2+\|W_2P\|_F^2.
\]
Under projected gradient flow,
\[
\dot W_1=-G_1+\alpha_1W_1,
\qquad
\dot W_2=-G_2+\alpha_2W_2,
\]
where
\[
\alpha_1:=\frac{\langle G_1,W_1\rangle_F}{\|W_1\|_F^2},
\qquad
\alpha_2:=\frac{\langle G_2,W_2\rangle_F}{\|W_2\|_F^2}.
\]
Since $\|W_i\|_F=1$, we have $|\alpha_i|\le \|G_i\|_F$. The same argument as in \Cref{lem:projected-gradient-U} gives $\|G_i\|_F\le A$, so $|\alpha_i|\le A$. Therefore,
\begin{align*}
\dot M_{\mathcal U}
&= 2\langle PW_1,P\dot W_1\rangle_F + 2\langle W_2P,\dot W_2P\rangle_F\\
&\le 2m_1\|PG_1\|_F + 2m_2\|G_2P\|_F + 2|\alpha_1|m_1^2 + 2|\alpha_2|m_2^2\\
&\le 2A m_1m_2 + 2A m_1m_2 + 2A m_1^2 + 2A m_2^2\\
&\le 4A(m_1^2+m_2^2)=4A M_{\mathcal U}.
\end{align*}
\end{proof}

\begin{corollary}\label{cor:mass-envelope}
For every fixed hidden subspace $\mathcal U$,
\[
M_{\mathcal U}(t) \le M_{\mathcal U}(0)e^{4At}.
\]
Equivalently,
\[
\sqrt{M_{\mathcal U}(t)} \le \sqrt{M_{\mathcal U}(0)}\,e^{2At}.
\]
\end{corollary}

\begin{proof}
Apply Gr\"onwall's inequality to \Cref{thm:weak-growth}.
\end{proof}

\main*

\begin{proof}
If the target is reached at time $T_\Delta$, then we have
\[
M_{\mathcal U_\Delta}(W(T_\Delta))\ge m.
\]
Combining with \Cref{cor:mass-envelope},
\[
m\leq M_{\mathcal U}(W(T_\Delta))
\le
M_{\mathcal U}(W(0))e^{4AT_\Delta}.
\]
Rearranging, we have 
\begin{align*}
 T &\geq \frac{1}{4A} \log\left(\frac{m}{M_{\mathcal U}(W(0))} \right) \\
 &\gtrsim \log\left(\frac{m}{n_1\sigma_{1,\mathcal U} + n_2\sigma_{2,\mathcal U}} \right),
\end{align*}
where the second inequality used \Cref{prop:initial-singular-mass}.

For the second case, we note that this corresponds to the setting where $\sigma_{1,\mathcal{U}}$ is the smallest singular value of $W_1$, and likewise for $\sigma_{2,\mathcal{U}}$ and $W_2$. Furthermore, we have $n_1=n_2=1$. We can then note that $\frac{1}{\sigma_{1,\mathcal{U}}+ \sigma_{2,\mathcal{U}}}\geq \frac12\min\{\frac{1}{\sigma_{1,\mathcal{U}}}, \frac{1}{\sigma_{2,\mathcal{U}}}\}$. Since both weight matrices have Frobenius norm 1, we further have $\frac{1}{\sigma_{1,\mathcal U}} \geq \kappa_1$, since the largest singular value of $W_1$ is at most 1. Combining these results, we obtain that
\[\frac{m}{\sigma_{1,\mathcal U} + \sigma_{2,\mathcal U}} \gtrsim m\cdot\min(\kappa_1,\kappa_2),\]
and we are complete.
\end{proof}

\section{Proofs for SingularClip}\label{sec:singularclip-proofs}

\begin{proposition}[Frobenius projection]\label{prop:frobenius-projection}
The matrix $\mathrm{sc}_a^b(W)$ is the Frobenius norm projection of $W$ onto the set of matrices with singular values in $[a, b]$. Equivalently, $\|\mathrm{sc}_a^b(W) - W\|_F \le \|B - W\|_F$ for any matrix $B$ with all singular values in $[a, b]$.
\end{proposition}
\begin{proof}
We seek to minimize $\|B - W\|_F^2$ subject to the singular values of $B$ lying in $[a, b]$.

Expanding the squared norm: $\|B - W\|_F^2 = \|B\|_F^2 + \|W\|_F^2 - 2\langle B, W\rangle_F$.

Let $W = U\Sigma V^\top$ be the SVD of $W$ with singular values $\sigma_1 \ge \dots \ge \sigma_r$. Let $B$ have singular values $\lambda_1 \ge \dots \ge \lambda_r$. Note that $\|B\|_F^2 = \sum \lambda_i^2$.

To minimize the distance, we must maximize the inner product $\langle B, W\rangle_F$. By the Von Neumann trace inequality for real matrices, $\langle B, W\rangle_F \le \sum_{i=1}^r \lambda_i \sigma_i$, with equality strictly achieved when $B$ is constructed using a simultaneous SVD alignment with $W$ (i.e., setting $B = U \operatorname{diag}(\lambda) V^\top$).

Thus, the optimization problem separates into scalar problems:
\[ \min_{\lambda} \sum_{i=1}^r (\lambda_i^2 + \sigma_i^2 - 2\lambda_i \sigma_i) = \min_{\lambda} \sum_{i=1}^r (\lambda_i - \sigma_i)^2, \]
subject to $\lambda_i \in [a, b]$.

Since the constraints are decoupled, we minimize each quadratic term $f(\lambda_i) = (\lambda_i - \sigma_i)^2$ independently subject to the bounds $\lambda_i \in [a, b]$. Because $f(\lambda_i)$ is a strictly convex parabola centered at $\sigma_i$, the global minimum on the restricted interval depends on the location of the vertex relative to the boundaries. If the vertex lies within the interval ($\sigma_i \in [a, b]$), the unconstrained minimum is achievable, natively yielding $\lambda_i^* = \sigma_i$. If the vertex lies strictly below the interval ($\sigma_i < a$), the function is strictly increasing on $[a, b]$, placing the minimal distance geometrically at the lower boundary $\lambda_i^* = a$. Conversely, if the vertex lies strictly above the interval ($\sigma_i > b$), the function is strictly decreasing on $[a, b]$, placing the minimum at the upper boundary $\lambda_i^* = b$. This piecewise optimality condition is evaluated by the continuous clamping function $\lambda_i^* = \min(\max(\sigma_i, a), b)$.

Because the function $z \mapsto \min(\max(z, a), b)$ is non-decreasing and the original $\sigma_i$ are ordered, the resulting $\lambda_i^*$ satisfy $\lambda_1^* \ge \dots \ge \lambda_r^*$. This confirms that the constructed matrix $B^*$ satisfies the SVD ordering requirements.

Because the constraint set is non-convex (due to the lower bound $a > 0$), the projection is not guaranteed to be globally unique when $W$ possesses degenerate or zero singular values, as the choice of orthonormal matrices $U$ and $V$ spanning those specific eigenspaces is not unique. However, the theoretical maximum inner product achieved via the Von Neumann trace inequality remains invariantly $\sum_{i=1}^r \lambda_i \sigma_i$ regardless of the chosen orthogonal basis within those subspaces. Thus, any valid SVD of $W$ provides an optimal rotational alignment.

Consequently, the constructed optimal matrix $B^* = U \operatorname{diag}(\lambda^*) V^\top$ corresponds exactly to the definition of the SingularClip operator $\mathrm{sc}_a^b(W)$. It adheres to the target constraint set and achieves the minimal Frobenius distance, confirming $\mathrm{sc}_a^b(W)$ as a valid optimal projection.

Finally, to establish the vector bound claimed in the main text, consider any input vector $x$. By the standard property relating the spectral norm (operator norm) to the Frobenius norm, we have:
\[ \|Wx - \mathrm{sc}_a^b(W)x\|_2 = \|(W - \mathrm{sc}_a^b(W))x\|_2 \le \|W - \mathrm{sc}_a^b(W)\|_{op} \|x\|_2 \le \|W - \mathrm{sc}_a^b(W)\|_F \|x\|_2. \]
This completes the proof.
\end{proof}